\documentclass{article}

\usepackage{microtype}
\usepackage{graphicx}
\usepackage{subfigure}
\usepackage{booktabs} % for professional tables

\usepackage[titletoc,title]{appendix}

\usepackage{amsmath}

\usepackage{amsthm}
\newtheorem{defn}{Definition}
\newtheorem{lem}{Lemma}
\newtheorem{thm}[lem]{Theorem}
\newtheorem{cor}[lem]{Corollary}

\newtheorem{prp}[lem]{Proposition}

\makeatletter
\newtheorem*{ref@theorem}{\ref@title}
\newcommand{\newreftheorem}[2]{%
\newenvironment{ref#1}[1]{%
 \def\ref@title{#2 \ref{##1}}%
 \begin{ref@theorem}}%
 {\end{ref@theorem}}}
\makeatother
\newreftheorem{thm}{Theorem}
\newreftheorem{lem}{Lemma}
\newreftheorem{prp}{Proposition}
\newreftheorem{cor}{Corollary}

\usepackage{amssymb}

\usepackage{pgfplots}
\usepackage{pgfplotstable}

\usepackage{hyperref}

\usepackage[accepted]{icml2018}

\icmltitlerunning{Bounding and Counting Linear Regions of Deep Neural Networks}

\definecolor{color1}{HTML}{E41A1C}
\definecolor{color2}{HTML}{377DB8}
\definecolor{color3}{HTML}{4DAF4A}
\definecolor{color4}{HTML}{984EA3}
\definecolor{color5}{HTML}{FF7F00}
\definecolor{color6}{HTML}{A65628}
\definecolor{color2b}{HTML}{98BFE0}

\begin{document}

\twocolumn[
\icmltitle{Bounding and Counting Linear Regions of Deep Neural Networks}



\icmlsetsymbol{equal}{*}

\begin{icmlauthorlist}
\icmlauthor{Thiago Serra}{equal,cmu}
\icmlauthor{Christian Tjandraatmadja}{equal,cmu}
\icmlauthor{Srikumar Ramalingam}{ut}
\end{icmlauthorlist}

\icmlaffiliation{cmu}{Carnegie Mellon University, Pittsburgh, USA}
\icmlaffiliation{ut}{The University of Utah, Salt Lake City, USA}

\icmlcorrespondingauthor{Thiago Serra}{tserraaz@alumni.cmu.edu}
\icmlcorrespondingauthor{Christian Tjandraatmadja}{ctjandra@alumni.cmu.edu}
\icmlcorrespondingauthor{Srikumar Ramalingam}{srikumar@cs.utah.edu}

\icmlkeywords{Machine Learning, ICML}

\vskip 0.3in
]



\printAffiliationsAndNotice{\icmlEqualContribution} % otherwise use the standard text.

%\begin{abstract}
%This document provides a basic paper template and submission guidelines.
%Abstracts must be a single paragraph, ideally between 4--6 sentences long.
%Gross violations will trigger corrections at the camera-ready phase.
%\end{abstract}

\begin{abstract}
We investigate the complexity of deep neural networks ({\sc DNN}) that represent piecewise linear ({\sc PWL}) functions. 
In particular, we study the number of linear regions, i.e. pieces, that a {\sc PWL} function represented by a {\sc DNN} can attain, both theoretically and empirically.
We present (i) tighter upper and lower bounds for the maximum number of linear regions on rectifier networks, which are exact for inputs of dimension one; 
(ii) a first upper bound for multi-layer maxout networks; and 
(iii) a first method to perform exact enumeration or counting of the number of regions by modeling the DNN with a mixed-integer linear formulation. 
These bounds come from leveraging the dimension of the space defining each linear region.
The results also indicate that a deep rectifier network can only have more linear regions than every shallow counterpart with same number of neurons if that number exceeds the dimension of the input.
\end{abstract}
\section{Introduction}

We have seen an unprecedented success of {\sc DNN}s in computer vision, speech, and other domains~\citep{Krizhevsky2012,Ciresan2012,Goodfellow2013,Hinton2012}. While the popular networks such as AlexNet~\citep{Krizhevsky2012}, GoogleNet~\citep{Szegedy2015}, and residual networks~\citep{He2016DeepRL} have shown record beating performance on various image recognition tasks, empirical results still govern the design of network architecture in terms of depth and activation functions. Two important considerations that are part of most successful architectures are greater depth and the use of {\sc PWL} activation functions such as rectified linear units (ReLUs). This large gap between practice and theory has driven researchers toward mathematical modeling of the expressive power of {\sc DNN}s~\citep{Cybenko1989,Anthony1999,Pascanu2013,Montufar2014,Bianchini2014,Eldan2015,Telgarsky2015,Mhaskar2016,Raghu2017,Montufar2017}. 

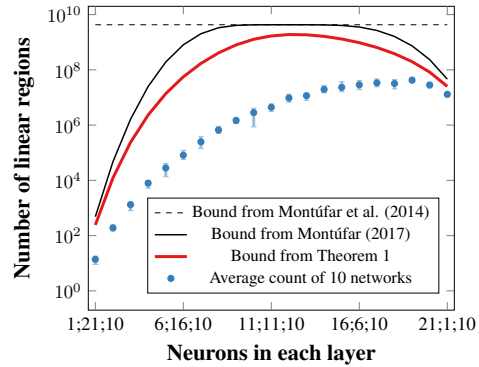
\begin{figure}[!t]
  \centering
\pgfplotsset{every error bar/.style={ultra thick, color=color2b}}
\begin{tikzpicture}[scale=0.71]
\begin{semilogyaxis}[filter discard warning=false, enlargelimits=false, enlarge x limits=0.02, 
%axis y line*=left,
%height=6cm,
%width=9cm,
%title={ADD TITLE},
title style={at={(0.5,-0.25)},anchor=north,yshift=-0.1},
xlabel={\large \textbf{Neurons in each layer}},
symbolic x coords={
1;21;10,
2;20;10,
3;19;10,
4;18;10,
5;17;10,
6;16;10,
7;15;10,
8;14;10,
9;13;10,
10;12;10,
11;11;10,
12;10;10,
13;9;10,
14;8;10,
15;7;10,
16;6;10,
17;5;10,
18;4;10,
19;3;10,
20;2;10,
21;1;10,
},
xtick=\empty,
extra x ticks={1;21;10,6;16;10,11;11;10,16;6;10,21;1;10,},
 %extra x tick labels={L100, L50P50, P100, P50B50, B100, B50L50, L100, L50C50, C100, C50B50, B100, B50C50, C100, C50P50, P100},
 %extra x tick style={
 %          grid=major,
  %         tick label style={rotate=90} % <-- this is added
 %          },
ylabel={\large \textbf{Number of linear regions}},ymin=2e-1,ymax=2e10, 
every axis y label/.style={at={(ticklabel cs:0.5)},rotate=90,anchor=near ticklabel}, 
legend style={at={(0.56,0.37)},anchor=north}
]
\pgfplotstableread{
x         y    ymin  ymax
1;21;10    13.8    9    18
2;20;10    190.5    147    243
3;19;10    1316.2    800    1639
4;18;10    7932.3    5188    9717
5;17;10    28250.3    13512    41228
6;16;10    81478.5    57220    121426
7;15;10    248055    141403    380276
8;14;10    665167.4    502841    897485
9;13;10    1474984.1    1128238    1804370
10;12;10    2839365.5    856141    4102088
11;11;10    4431332.8    3090088    5858622
12;10;10    9612998.4    6703494    13131100
13;9;10    11616563.3    7889964    14122481
14;8;10    19621470    14892942    27586611
15;7;10    23220422    16096760    37544427
16;6;10    28800096.9    19388029    41041660
17;5;10    33969273.1111111    24302711    46776349
18;4;10    32112914.2    20425922    44773930
19;3;10    42112718.5555556    35810115    54015210
20;2;10    28282865.5    22193955    35171688
21;1;10    13071698.4    10460528    14655435
}{\mytable}
\addplot[dashed] coordinates {
(1;21;10,4294967296)
(21;1;10,4294967296)
};
\addplot[thick] coordinates {
(1;21;10,484)
(2;20;10,47264)
(3;19;10,1633280)
(4;18;10,25000448)
(5;17;10,191951232)
(6;16;10,808272896)
(7;15;10,2030043136)
(8;14;10,3348183808)
(9;13;10,4092785664)
(10;12;10,4281335808)
(11;11;10,4294967296)
(12;10;10,4294967296)
(13;9;10,4290772992)
(14;8;10,4248829952)
(15;7;10,4060086272)
(16;6;10,3556769792)
(17;5;10,2675965952)
(18;4;10,1619001344)
(19;3;10,738197504)
(20;2;10,234881024)
(21;1;10,46137344)
};
\addplot[ultra thick,color=color1] coordinates {
(1;21;10,243)
(2;20;10,12279)
(3;19;10,236909)
(4;18;10,2316709)
(5;17;10,13756567)
(6;16;10,56128117)
(7;15;10,171071287)
(8;14;10,411552217)
(9;13;10,800917467)
(10;12;10,1283052848)
(11;11;10,1690286436)
(12;10;10,1902816995)
(13;9;10,1858910222)
(14;8;10,1636341897)
(15;7;10,1312054984)
(16;6;10,965299552)
(17;5;10,645713191)
(18;4;10,385283875)
(19;3;10,198153450)
(20;2;10,82836506)
(21;1;10,25165813)
};
\addplot+[only marks, mark=*,mark size=1.5,color=color2] 
  plot[thick, error bars/.cd, y dir=both, y explicit]
  table[x=x,y=y,y error plus expr=\thisrow{ymax}-\thisrow{y},y error minus expr=\thisrow{y}-\thisrow{ymin}] {\mytable};
\legend{\small Bound from Mont\'ufar et al. (2014),\small Bound from Mont\'ufar (2017),\small Bound from Theorem~1,\small Average count of 10 networks}
\end{semilogyaxis}
\end{tikzpicture}
\caption{\emph{This paper shows improved bounds on the "number of linear regions" (typically used to study the expressiveness of {\sc DNN}s) of {\sc PWL} functions modeled by {\sc DNN}s that use rectified linear activation functions, and a method for exact counting of the number of such regions in trained networks. We compare upper bounds from the first and latest results in the literature~\cite{Montufar2014,Montufar2017} with our main result (See Theorem~\ref{thm:upper_bound_improved}). Using the proposed exact counting algorithm, we show the  
actual 
number of linear regions in 10 rectifier networks for MNIST digit recognition task with each configuration of two hidden layers totaling 22 neurons, reporting average and min-max range.}}
\label{fig:bounds_counting}
\end{figure}

The expressiveness of {\sc DNN}s can be studied by transforming one network to another with different number of layers or activation functions. 
While any continuous function can be modeled using a single hidden layer of sigmoid activation functions~\citep{Cybenko1989}, shallow networks require exponentially more neurons to model functions that can be modeled using much smaller deeper networks~\citep{Delalleau2011}. 
There are a wide variety of activation functions, which come with different modeling capabilities, such as threshold ($f(z) = (z>0)$),  logistic ($f(z) = 1/(1+\exp(-e))$),  ReLU  ($f(z) = \max\{0,z\}$), and maxout ($f(z_1,z_2,\dots,z_k)=\max\{z_1,z_2,\dots,z_k\}$). It has been shown that sigmoid networks are more expressive than similar-sized threshold networks~\citep{Maass1994}, and ReLU networks are more expressive than similar-sized threshold networks~\citep{Pan2016}. 

The complexity or expressiveness of neural networks belonging to the family of {\sc PWL} functions can also be analyzed by looking at how the network can partition the input space to an exponential number of linear response regions~\citep{Pascanu2013,Montufar2014}. The basic idea of a {\sc PWL} function is simple: we can divide the input space into several regions and we have individual linear functions for each of these regions. Functions partitioning the input space to a larger number of linear regions are considered to be more complex, or in other words, possess better representational power. In the case of ReLUs, it has been shown that some deep networks separate their input space into exponentially more linear response regions than their shallow counterparts despite using the same number of activation functions~\citep{Pascanu2013}. The results were later extended and improved. In particular,~\citet{Montufar2014} show upper and lower bounds on the maximal number of linear regions for a ReLU {\sc DNN} and a single layer maxout network, and a lower bound for a maxout {\sc DNN}. Furthermore, \citet{Raghu2017} and \citet{Montufar2017} improve the upper bound for a ReLU {\sc DNN}. This upper bound asymptotically matches the lower bound from~\citet{Montufar2014} when the number of layers and input dimension are constant and all layers have the same width. Finally, \citet{Arora2018} improve the lower bound by providing a family of ReLU {\sc DNN}s with an exponential number of regions for fixed size and depth. 

\noindent
{\bf Main Contributions}

This paper directly improves on the results of Mont\'{u}far et al.~\citep{Pascanu2013,Montufar2014,Montufar2017}, \citet{Raghu2017}, and \citet{Arora2018}. Fig.~\ref{fig:bounds_counting} highlights the main contributions, and the following list summarizes all the contributions:
\begin{itemize}
\item We achieve tighter upper and lower bounds on the maximal number of linear regions of the {\sc PWL} function corresponding to a {\sc DNN} that employs ReLUs as shown in Fig.~\ref{fig:bounds_counting}. As a special case, we present the exact maximal number of regions when the input dimension is one. We additionally provide the first upper bound 
for multi-layer maxout networks. (See Sections~\ref{sec:bound_rectifier} and ~\ref{sec:bound_maxouts}).\newline
\item We show for ReLUs that the exact maximal number of linear regions of shallow networks is larger than that of deep networks if the input dimension exceeds the number of neurons, a result that could not be inferred from bounds in prior work. (See Figure~\ref{fig:insightsB}). \newline
\item We use a mixed-integer linear formulation to show that counting linear regions is indeed possible. For the first time, we show the exact number of linear regions for a sample of 
{\sc DNN}s as shown in Fig.~\ref{fig:bounds_counting}.  
This new capability can be used to evaluate the tightness of the bounds and potentially to analyze the correlation between accuracy and the number of linear regions. 
(See Sections~\ref{sec:counting} and ~\ref{sec:exp}). 
\end{itemize}
\section{Notations and Background}\label{sec:background}

Let us assume that a feedforward neural network, which is studied in this paper, has $n_0$ input variables given by $\mathbf{x} = \{x_1,x_2,\dots,x_{n_0}\}$, and $m$ output variables given by $\mathbf{y} = \{y_1,y_2,\dots,y_m\}$. Each hidden layer $l=\{1,2,\dots,L\}$ has $n_l$ hidden neurons whose activations are given by $\mathbf{h}^l = \{h_1^l,h_2^l,\dots,h_{n_l}^l\}$. Let $W^l$ be the $n_l \times n_{l-1}$ matrix where each row corresponds to the weights of a neuron of layer $l$. Let $\mathbf{b}^l$ be the bias vector used to obtain the activation functions of neurons in layer $l$. 
Based on the $\operatorname{ReLU}(x) = \max\{0,x\}$ activation function, the activations of the hidden neurons and the outputs are given below:
\begin{eqnarray*}
\mathbf{h}^1 & = &  \max\{0,W^1 \mathbf{x} + b^1\} \\ 
\mathbf{h}^l & = &  \max\{0,W^l \mathbf{h}^{l-1} + b^l\} \\
\mathbf{y}   & = &  W^{L+1} \mathbf{h^L}
\end{eqnarray*}
As considered in~\citet{Pascanu2013}, the output layer is a linear layer that computes the linear combination of the activations from the previous layer without any ReLUs.

We can treat the {\sc DNN} as a piecewise linear ({\sc PWL}) function $F:\mathbb{R}^{n_0}\rightarrow \mathbb{R}^{m}$ that maps the input $\mathbf{x}$ in $\mathbb{R}^{n_0}$ to $\mathbf{y}$ in $\mathbb{R}^m$. 
We define linear regions based on activation patterns.

\noindent
{\bf Activation Pattern:} Let us consider an input vector $\mathbf{x} = \{x_1,\ldots,x_{n_0}\}$. For every layer $l$ we define an activation set $S^l \subseteq \{1,\ldots, n_l\}$ such that $e \in S^l$ if and only if the ReLU $e$ is active, i.e. $h^l_e > 0$. We aggregate these activation sets into $\mathcal{S} = (S^1, \ldots, S^{l})$, which we call an activation pattern. In this work, we may consider activation patterns up to a layer $l \leq L$. Activation patterns were previously defined in terms of strings~\citep{Raghu2017,Montufar2017}. We say that an input $\mathbf{x}$ corresponds to an activation pattern $\mathcal{S}$ 
if feeding $\mathbf{x}$ to the DNN results in the activations in $\mathcal{S}$.

We define linear regions as follows.\footnote{There is a subtly different definition of linear region in the literature~\citep{Pascanu2013,Montufar2014}, as follows. Given a {\sc PWL} function $F$, a linear region is a maximal connected subset of the input space on which $F$ is linear. The two definitions are essentially the same, except in a few degenerate cases. There could be scenarios where two different activation patterns may correspond to two adjacent regions with the same linear function, in which case this definition considers them as a single one. However, the bounds derived in this paper are valid for both definitions.}

\begin{defn}
Given a {\sc PWL} function $F:\mathbb{R}^{n_0}\rightarrow \mathbb{R}^{m}$ represented by a DNN, a linear region is the set of 
inputs 
that correspond to a same activation pattern  
in the DNN.
\label{def:linear_regions2}
\end{defn}

In this paper, we interchangeably refer to $\mathcal{S}$ as an activation pattern or a region for convenience.

In Figure~\ref{fig:regions_2d} we show a simple ReLU {\sc DNN} with two inputs $\{x_1,x_2\}$ and 3 layers.
The activation units $\{a,b,c,d,e,f\}$ on these layers can be thought of as hyperplanes that each divide the space in two. On one side of the hyperplane, the unit outputs a positive value. For all points on the other side of the hyperplane including itself, the unit outputs 0.

\begin{figure}[tp]
  \centering
    \includegraphics[width=\columnwidth]{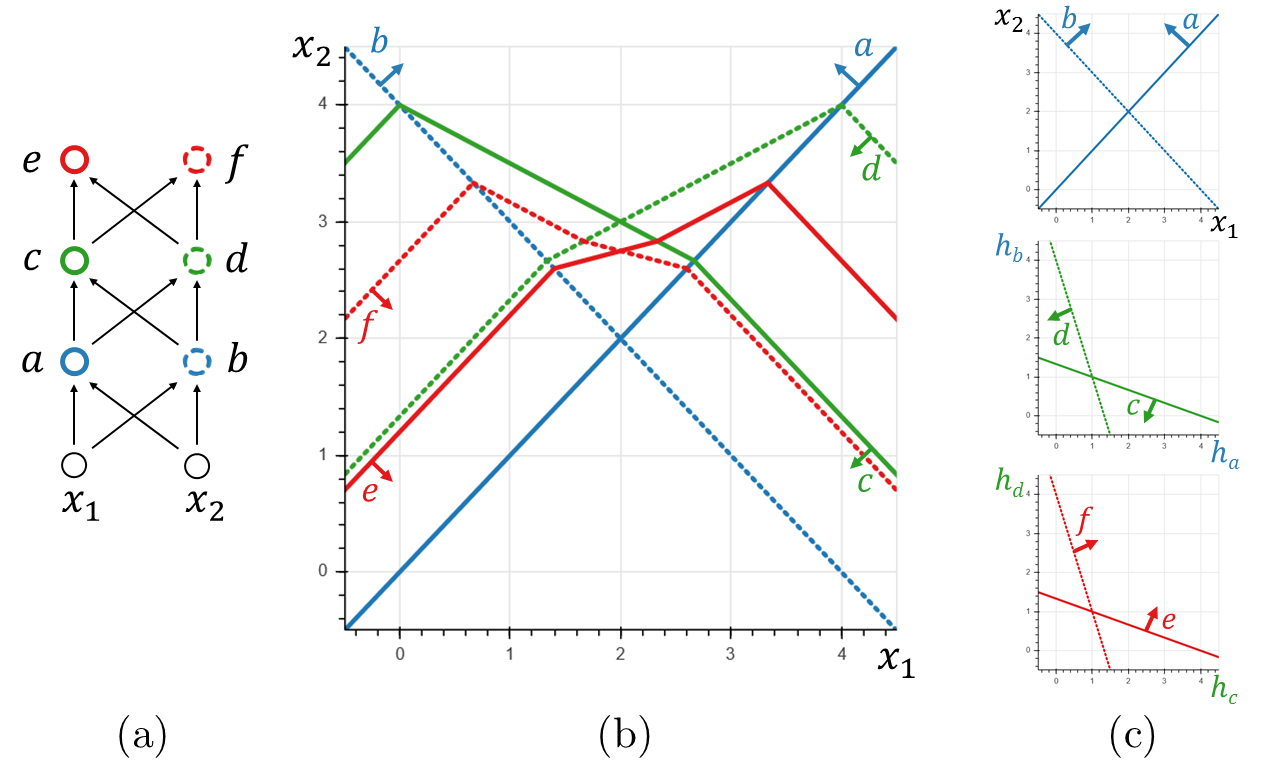}
\caption{\emph{(a) Simple {\sc DNN} with two inputs and three layers with 2 activation units each. (b) Visualization of the hyperplanes from the layers partitioning the input space into 20 linear regions. The arrows indicate the directions in which the corresponding neurons are activated. (c) Visualization of the hyperplanes from the first, second, and third (non-input) layers (from top to bottom) in the space given by the outputs of their respective previous layers. The outputs of the units are $h_a = \max\{0, -x_1 + x_2\}$, $h_b = \max\{0, x_1 + x_2 - 4\}$, $h_c = \max\{0, -h_a -3h_b + 4\}$, $h_d = \max\{0, -3h_a - h_b + 4\}$, $h_e = \max\{0, h_c + 3h_d - 4\}$, and $h_f = \max\{0, 3h_c + h_d - 4\}$.}}
	\label{fig:regions_2d}
\end{figure}

One may wonder: into how many linear regions do $n$ hyperplanes split a space? \citet{Zaslavsky1975} shows that an arrangement of $n$ hyperplanes divides a $d$-dimensional space into at most $\sum_{s=0}^d {n \choose s}$ regions, a bound that is attained when they are in general position. The term general position basically means 
that a small perturbation of the hyperplanes does not change the number of regions. This corresponds to the exact maximal number of regions of a single layer {\sc DNN} with $n$ ReLUs and input dimension $d$.

In Figures~\ref{fig:regions_2d}(b)--(c), we provide a visualization of how ReLUs partition the input space. Figure~\ref{fig:regions_2d}(c) shows the hyperplanes corresponding to the ReLUs at layers $l = 1, 2,$ and $3$, from top to bottom. Figure~\ref{fig:regions_2d}(b) considers these same hyperplanes in the input space $x$. If we consider only the first-layer hyperplanes, the 2D input space is partitioned into 4 regions, as per \citet{Zaslavsky1975} $\left( {2 \choose 0} + {2 \choose 1} + {2 \choose 2} = 4 \right)$. The regions are further partitioned as we consider additional layers. Subsequent hyperplanes are affected by the transformations applied in the earlier layers.

Figure~\ref{fig:regions_2d} also highlights that activation boundaries behave like hyperplanes when inside a region and may bend whenever they intersect with a boundary from a previous layer. This has also been pointed out by~\citet{Raghu2017}. In particular, they cannot appear twice in the same region as they are defined by a single hyperplane if we fix the region. Moreover, these boundaries do not need to be connected, as illustrated in Figure~\ref{fig:regions_1d}.

\begin{figure}[tp]
  \centering
    \includegraphics[width=\columnwidth]{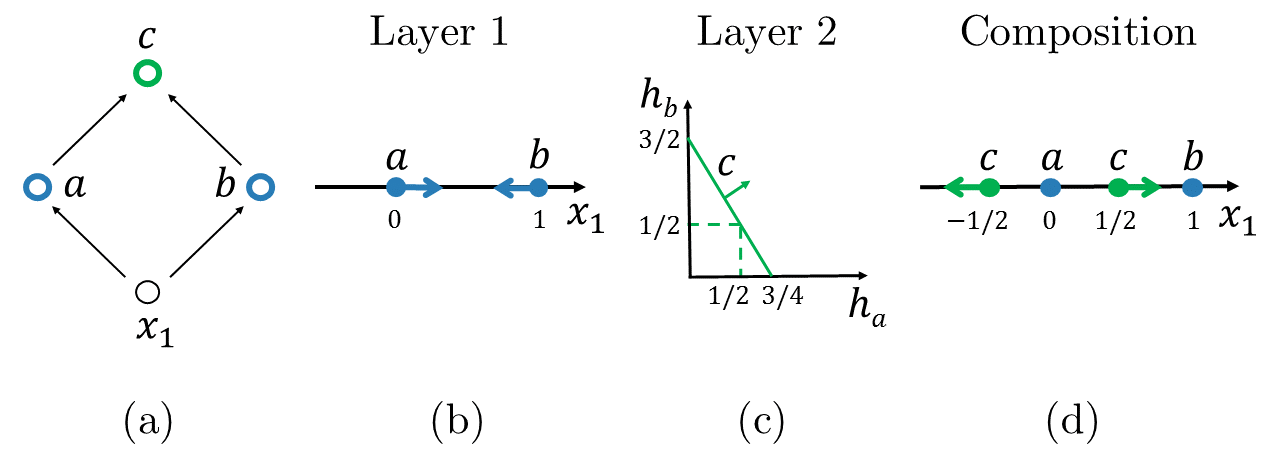}
  \caption{\emph{(a) A network with one input $x_1$ and three activation units $a$, $b$, and $c$. (b) We show the hyperplanes $x_1 = 0$ and $-x_1 + 1 = 0$ corresponding to the two activation units in the first hidden layer. In other words, the activation units are given by $h_a = \max\{0,x_1\}$ and $h_b = \max\{0,-x_1 + 1\}$. (c) The activation unit in the third layer is given by $h_c = \max\{0, 4h_a + 2h_b - 3\}$. (d) The activation boundary for neuron $c$ is disconnected.}}
	\label{fig:regions_1d}
\end{figure}
%!TEX root = bounding_counting.tex

\section{Tighter Bounds for Rectifier Networks}\label{sec:bound_rectifier}

\citet{Montufar2014} derive an upper bound of $2^N$ for $N$  
units, which can be obtained by mapping linear regions to activation patterns. \citet{Raghu2017} improve this result by deriving an asymptotic upper bound of $O(n^{Ln_0})$ to the maximal number of regions, assuming $n_l = n$ for all layers $l$ and $n_0 = O(1)$. \citet{Montufar2017} further tightens the upper bound to $\prod_{l=1}^L \sum_{j=0}^{d_l} \binom{n_l}{j}$, where $d_l = \min\{n_0, n_1, \ldots, n_l\}$.

Moreover, \citet{Montufar2014} prove a lower bound of $\left(\prod_{l=1}^{L-1}\lfloor n_l/n_0 \rfloor^{n_0}\right) \sum_{j=0}^{n_0} \binom{n_L}{j}$ when $n \geq n_0$, or asymptotically $\Omega((n/n_0)^{(L-1)n_0} n^{n_0})$. \citet{Arora2018} present a lower bound of $2\sum_{j=0}^{n_0-1}\binom{m-1}{j}w^{L-1}$ where $2m = n_1$ and $w = n_l$ for all $l = 2, \ldots, L$. We derive both upper and lower bounds that improve upon these previous results.

\subsection{An Upper Bound on the Number of Linear Regions}\label{sec:upper_bound_rectifier}
In this section, we prove the following upper bound on the number of regions.

\begin{thm}\label{thm:upper_bound_improved}
Consider a deep rectifier network with $L$ layers, $n_l$ rectified linear units at each layer $l$, and an input of dimension $n_0$. The maximal number of regions of this neural network is at most
\begin{align*}
\sum_{(j_1,\ldots,j_L) \in J} \prod_{l=1}^L \binom{n_l}{j_l}
\end{align*}
where $J = \{(j_1, \ldots, j_L) \in \mathbb{Z}^L: 0 \leq j_l \leq \min\{n_0, n_1 - j_1, \ldots, n_{l-1} - j_{l-1}, n_l\}\ \forall l = 1, \ldots, L\}$. This bound is tight when $L = 1$.
\end{thm}

Note that this is a stronger upper bound than the one that appeared in~\citet{Montufar2017}, which can be derived from this bound by relaxing the terms $n_l - j_l$ to $n_l$ and factoring the expression. When $n_0 = O(1)$ and all layers have the same width $n$, 
we have 
the same best known asymptotic bound $O(n^{Ln_0})$ first presented in~\citet{Raghu2017}. 

Two insights can be extracted from the above expression:

\begin{enumerate}
\item{\textbf{Bottleneck effect.} The bound is sensitive to the positioning of layers that are small relative to the others, a property we call the bottleneck effect. If we subtract a neuron from one of two layers with the same width, choosing the one closer to the input layer will lead to a larger (or equal) decrease in the bound. This occurs because each index $j_l$ is essentially limited by the widths of the current and previous layers, $n_0, n_1, \ldots, n_l$. In other words, smaller widths in the first few layers of the network imply a bottleneck on the bound.

The following proposition illustrates this bottleneck effect of the bound for the 2-layer case (see Appendix~\ref{sec:analysis_bound} for the proof).

\begin{prp}\label{prp:bottleneck_2_layer}
Consider a 2-layer network with widths $n_1, n_2$ and input dimension $n_0 > \max\{n_1, n_2\}$. Then moving a neuron from the first layer to the second layer strictly decreases the bound.
\end{prp}

Figure~\ref{fig:insightsA} illustrates this behavior. For the solid line, we keep the total size of the network the same but shift from a small-to-large network (i.e., smaller width near the input layer and larger width near the output layer) to a large-to-small network in terms of width. We see that the bound monotonically increases as we reduce the bottleneck. If we add a layer of constant width at the end, represented by the dashed line, the bound decreases when the layers before the last become too small and create a bottleneck for the last layer.

While this is a property of the upper bound rather than of the exact maximal number of regions, 
we observe in Section~\ref{sec:exp} that empirical results for the number of regions of a trained network exhibit a behavior that resembles the bound as the width of the layers vary. Moreover, this bottleneck effect appears at a more fundamental level. For example, having a first layer of size one forces all hyperplanes corresponding to subsequent layers to be parallel to each other in the input space, reflecting the fact that we have compressed all information into a single value. More generally, the smaller a layer is, the more linearly dependent the hyperplanes from subsequent layers will be, which results in fewer regions. Further in this section, we formalize this in terms of dimension and show that the dimension of the image of a region is limited by the widths of earlier layers, which is used to prove Theorem~\ref{thm:upper_bound_improved}.
}

\item{\textbf{Deep vs shallow for large input dimensions.} In several applications such as imaging, the input dimension can be very large. \citet{Montufar2014} show that if the input dimension $n_0$ is constant, then the number of regions of deep networks is asymptotically larger than that of shallow (single-layer) networks. We complement this picture by establishing that if the input dimension is large, then shallow networks can attain more regions than deep networks.

More precisely, we compare a deep network with $L$ layers of equal width $n$ and a shallow network with one layer of width $Ln$. Denote the exact maximal number of regions by $R(n_0, n_1, \ldots, n_L)$, where $n_0$ is the input dimension and $n_1, \ldots, n_L$ are the widths of layers 1 through $L$ of the network.

\begin{cor}\label{cor:shallow_vs_deep}
Let $L \geq 2$, $n \geq 1$, and $n_0 \geq Ln$. Then
\begin{align*}
R(n_0, \underbrace{n, \ldots, n}_{\text{$L$ times}}) < R(n_0, Ln)
\end{align*}

Moreover, $\lim_{L \to \infty} \frac{R(n_0, n, \ldots, n)}{R(n_0, Ln)} = 0$.
\end{cor}

This is a consequence of Theorem~\ref{thm:upper_bound_improved} (see Appendix~\ref{sec:analysis_bound} for the proof).

Figure~\ref{fig:insightsB}~(a) illustrates this behavior. As we increase the number of layers while keeping the total size of the network constant, the bound plateaus at a value lower than the exact maximal number of regions for shallow networks. Moreover, the number of layers that yields the highest bound decreases as we increase the input dimension $n_0$.
Hence, 
for a given number of units and $n_0$, 
there is a particular depth maximizing the bound. 

This property cannot be inferred from upper bounds derived in prior work: previous bounds for a network with $L$ layers of size $n$ are no smaller than $R(n_0, Ln)$ for a sufficiently large $n_0$. Figure~\ref{fig:insightsB}~(b) shows the behavior of the previous best known bound.

We remark that asymptotically both deep and shallow networks can attain exponentially many regions when the input dimension is sufficiently large. More precisely, for a {\sc DNN} with the same width $n$ per layer, the maximal number of linear regions is $\Omega(2^{\frac{2}{3}Ln})$ when $n_0 \geq n / 3$ (see Appendix~\ref{sec:exp_number_regions}).
}
\end{enumerate}

\definecolor{color1}{HTML}{E41A1C}
\definecolor{color2}{HTML}{377DB8}
\definecolor{color3}{HTML}{4DAF4A}
\definecolor{color4}{HTML}{984EA3}
\definecolor{color5}{HTML}{FF7F00}
\definecolor{color6}{HTML}{A65628}

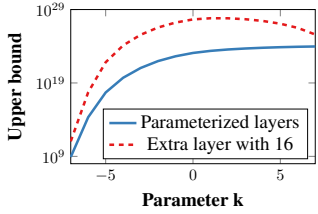
\begin{figure}[htp]
\centering
\begin{tikzpicture}[scale=0.6]
\begin{semilogyaxis}[filter discard warning=false, enlargelimits=false, 
height=5cm,
width=7cm,
xlabel={\large \textbf{Parameter k}},xmin=-7, xmax=7,
ylabel={\large \textbf{Upper bound}},ymin=1e8,ymax=1e29, 
every axis y label/.style={at={(ticklabel cs:0.5)},rotate=90,anchor=near ticklabel}, 
legend pos = south east%, legend style={fill=cmured,draw=white}
]
\addplot[ultra thick,color=color2] table [x index=0, y index=1, col sep=comma] {plot_data_input_size_A.txt} ;
\addplot[dashed,ultra thick, color=color1] table [x index=0, y index=2, col sep=comma] {plot_data_input_size_A.txt} ;
\legend{\large Parameterized layers,\large Extra layer with 16}
\end{semilogyaxis}
\end{tikzpicture}
\caption{\emph{Upper bound from Theorem~\ref{thm:upper_bound_improved}, in semilog scale, for input dimension $n_0 = 32$ and the width of the first five layers parameterized as $16+2k,16+k,16,16-k,16-2k$.}}\label{fig:insightsA}
\end{figure}

\begin{figure}[htp]
\centering
%  \label{fig:bound_plot_bottleneck}
\begin{tikzpicture}[scale=0.6]
\begin{axis}[filter discard warning=false, enlargelimits=false, 
height=5cm,
width=6cm,
title={\Large (a)~Theorem~\ref{thm:upper_bound_improved}},
title style={at={(0.5,-0.35)},anchor=north,yshift=-0.1},
xlabel={\large \textbf{Input dimension n$_0$}},xmin=1, xmax=42,
ylabel={\large \textbf{Upper bound}},
ymin=0,ymax=1.2e18, 
every axis y label/.style={at={(ticklabel cs:0.5)},rotate=90,anchor=near ticklabel}, 
%legend pos = outer north east%, legend style={fill=cmured,draw=white}
legend style={at={(1.1,0.5)},anchor=west}
]
\addplot[ultra thick,color=color1] table [x index=0, y index=1, col sep=comma] {plot_data_input_size_B.txt} ;
\addplot[ultra thick,color=color2] table [x index=0, y index=2, col sep=comma] {plot_data_input_size_B.txt} ;
\addplot[ultra thick,color=color3] table [x index=0, y index=3, col sep=comma] {plot_data_input_size_B.txt} ;
\addplot[ultra thick,color=color4] table [x index=0, y index=4, col sep=comma] {plot_data_input_size_B.txt} ;
\addplot[ultra thick,color=color5] table [x index=0, y index=5, col sep=comma] {plot_data_input_size_B.txt} ;
\addplot[ultra thick,color=color6] table [x index=0, y index=6, col sep=comma] {plot_data_input_size_B.txt} ;
% \legend{1x60,2x30,3x20,4x15,5x12,6x10}
\end{axis}
\end{tikzpicture}
\begin{tikzpicture}[scale=0.6]
\begin{axis}[filter discard warning=false, enlargelimits=false, 
height=5cm,
width=6cm,
title={\Large (b)~\citet{Montufar2017}},
title style={at={(0.5,-0.35)},anchor=north,yshift=-0.1},
xlabel={\large \textbf{Input dimension n$_0$}},xmin=1, xmax=42,
ymin=0,ymax=1.2e18, 
every axis y label/.style={at={(ticklabel cs:0.5)},rotate=90,anchor=near ticklabel}, 
legend style={at={(1.1,0.5)},anchor=west}
% legend pos = south east%, legend style={fill=cmured,draw=white}
]
\addplot[ultra thick,color=color1] table [x index=0, y index=1, col sep=comma] {plot_data_input_size_B-Montufar.txt} ;
\addplot[ultra thick,color=color2] table [x index=0, y index=2, col sep=comma] {plot_data_input_size_B-Montufar.txt} ;
\addplot[ultra thick,color=color3] table [x index=0, y index=3, col sep=comma] {plot_data_input_size_B-Montufar.txt} ;
\addplot[ultra thick,color=color4] table [x index=0, y index=4, col sep=comma] {plot_data_input_size_B-Montufar.txt} ;
\addplot[ultra thick,color=color5] table [x index=0, y index=5, col sep=comma] {plot_data_input_size_B-Montufar.txt} ;
\addplot[ultra thick,color=color6] table [x index=0, y index=6, col sep=comma] {plot_data_input_size_B-Montufar.txt} ;
\legend{1x60,2x30,3x20,4x15,5x12,6x10}
\end{axis}
\end{tikzpicture}
%  \label{fig:bound_plot_n0}
\caption{\emph{Comparison of bounds for evenly distributing 60 neurons in 1 to 6 layers, where the single-layer case is exact, according to the input dimension for (a)~Theorem~\ref{thm:upper_bound_improved} and (b)~\citet{Montufar2017}.}}\label{fig:insightsB}
\end{figure}
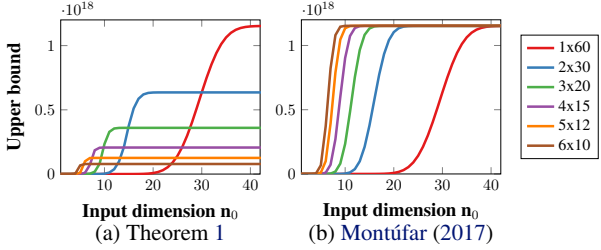

We now build towards the proof of Theorem~\ref{thm:upper_bound_improved}. For a given activation set $S^l$ and a matrix $W$ with $n_l$ rows, let $\sigma_{S^l}(W)$ be the operation that zeroes out the rows of $W$ that are inactive according to $S^l$. This represents the effect of the ReLUs. We say that a region is at layer $l$ if it corresponds to an activation pattern $(S^1, \ldots, S^l)$ up to layer $l$. For a region $\mathcal{S}$ at layer $l-1$, define $\bar{W}^l_{\mathcal{S}} := W^l \; \sigma_{S^{l-1}}(W^{l-1}) \cdots \sigma_{S^1}(W^1)$.

Each region $\mathcal{S}$ at layer $l-1$ may be partitioned by a set of hyperplanes defined by the neurons of layer $l$. When viewed in the input space, these hyperplanes are the rows of $\bar{W}^l_{\mathcal{S}}x + b = 0$ for some $b$. To verify this, note that, if we recursively substitute out the hidden variables $h_{l-1}, \ldots, h_1$ from the original hyperplane $W^l h_{l-1} + b_l = 0$ following $\mathcal{S}$, the resulting weight matrix applied to $x$ is $\bar{W}^l_{\mathcal{S}}$.

A central element in the proof of Theorem~\ref{thm:upper_bound_improved} is the dimension of the image of a linear region $\mathcal{S}$ under the DNN function $\mathbf{h}^{l}$ up to a layer $l$, which we denote by $\operatorname{dim}(\mathbf{h}^{l}(\mathcal{S}))$. For a fixed region $\mathcal{S}$ at layer $l$, 
that function 
is linear, and thus the dimension equals to the rank of the coefficient matrix. 
Therefore, $\operatorname{dim}(\mathbf{h}^{l}(\mathcal{S})) = \operatorname{rank}(\sigma_{S^{l}}(W^{l}) \cdots \sigma_{S^1}(W^1))$. This can be interpreted as the dimension of the space corresponding to $\mathcal{S}$ that the hyperplanes defined by $W^{l+1} h_l + b_{l+1} = 0$ effectively partitions. A key observation 
is that, once this dimension falls to a certain value, the regions contained in $\mathcal{S}$ at subsequent layers cannot recover to a higher dimension.

\citet{Zaslavsky1975} showed that the maximal number of regions in $\mathbb{R}^d$ induced by an arrangement of $m$ hyperplanes is at most $\sum_{j=0}^{d} \binom{m}{j}$. Moreover, this value is attained if and only if the hyperplanes are in general position. The lemma below tightens this bound for a special case where the hyperplanes may not be in general position.

\begin{lem}\label{lem:regions_rank}
Consider $m$ hyperplanes in $\mathbb{R}^d$ defined by the rows of $Wx + b = 0$. Then the number of regions induced by the hyperplanes is at most $\sum_{j=0}^{\operatorname{rank}(W)} \binom{m}{j}$.
\end{lem}
The proof is given in Appendix~\ref{sec:proof_regions_rank}. Its key idea is that it suffices to count regions within the row space of $W$. The next lemma brings Lemma~\ref{lem:regions_rank} into our context.

\begin{lem}\label{lem:regions_rank_network}
The number of regions induced by the $n_l$ neurons at layer $l$ within a certain region $\mathcal{S}$ is at most $\sum_{j=0}^{\min\{n_l, \operatorname{dim}(\mathbf{h}^{l-1}(\mathcal{S}))\}} \binom{n_l}{j}$.
\end{lem}
\begin{proof}
The hyperplanes in a region $\mathcal{S}$ of the input space are given by the rows of $\bar{W}^l_{\mathcal{S}}x + b = 0$ for some $b$. 
By definition, 
the rank of $\bar{W}^l_{\mathcal{S}}$ is upper bounded by $\min\{\operatorname{rank}(W^l), \operatorname{rank}(\sigma_{S^{l-1}}(W^{l-1}) \cdots \sigma_{S^{1}}(W^{1}))\} = \min\{\operatorname{rank}(W^l), \operatorname{dim}(\mathbf{h}^{l-1}(\mathcal{S}))\}$. That is, $\operatorname{rank}(\bar{W}^l_{\mathcal{S}}) \leq \min\{n_l, \operatorname{dim}(\mathbf{h}^{l-1}(\mathcal{S}))\}$, %Applying Lemma~\ref{lem:regions_rank} yields the~result.
and we apply Lemma~\ref{lem:regions_rank}. 
\end{proof}

In the next lemma, we show that the dimension of the image of a region $\mathcal{S}$ can be bounded recursively in terms of the dimension of the image of the region containing $\mathcal{S}$ and the number of activated neurons defining $\mathcal{S}$.

\begin{lem}\label{lem:dim_bound}
Let $\mathcal{S}$ be a region at layer $l$ and $\mathcal{S}'$ be the region at layer $l-1$ that contains it. Then $\operatorname{dim}(\mathbf{h}^{l}(\mathcal{S})) \leq \min\{|S^{l}|, \operatorname{dim}(\mathbf{h}^{l-1}(\mathcal{S}'))\}$.
\end{lem}
\begin{proof}
$\operatorname{dim}(\mathbf{h}^{l}(\mathcal{S})) = \operatorname{rank}(\sigma_{S^{l}}(W^{l}) \cdots \sigma_{S^1}(W^1)) \leq \min\{\operatorname{rank}(\sigma_{S^{l}}(W^{l})), \operatorname{rank}(\sigma_{S^{l-1}}(W^{l-1}) \cdots \sigma_{S^1}(W^1)) \leq \min\{|S^{l}|, \operatorname{dim}(\mathbf{h}^{l-1}(\mathcal{S}'))\}$. 
The last inequality comes from %the fact that the zeroed out rows do not count 
zeroed out rows not counting 
towards the %rank of the matrix.
matrix rank. 
\end{proof}

We now have the ingredients to prove Theorem~\ref{thm:upper_bound_improved}. 

\begin{proof}[Proof of Theorem~\ref{thm:upper_bound_improved}]
As illustrated in Figure~\ref{fig:regions_2d}, the partitioning can be 
regarded
as a sequential process: at each layer, we partition the regions obtained from the previous layer. 
When viewed in the input space, each region $\mathcal{S}$ obtained at layer $l-1$ is potentially partitioned by $n_l$ hyperplanes given by the rows of $\bar{W}^l_{\mathcal{S}}x + b = 0$ for some bias $b$. Some %of these 
hyperplanes may fall outside the interior of $\mathcal{S}$ and 
have~no~effect. 

With this process in mind, we recursively bound the number of subregions within a region. More precisely, we construct a recurrence $R(l,d)$ to 
%be an 
upper bound %to 
the maximal number of regions attainable from partitioning a region with image of dimension $d$ using units from layers $l, l+1, \ldots, L$. The base case %of the recurrence 
is given by Lemma~\ref{lem:regions_rank_network}: $R(L,d) = \sum_{j=0}^{\min\{n_L,d\}} \binom{n_L}{j}$. Based on Lemma~\ref{lem:dim_bound}, 
%we can write the recurrence by grouping 
the recurrence groups 
together regions with %the 
same activation set size $|S^l|$, as follows: $R(l,d) = \sum_{j=0}^{n_l} N_{n_l, d, j} R(l+1, \min\{j, d\})$ for all $l = 1, \ldots, L-1$. Here, $N_{n_l, d, j}$ represents the maximum number of regions with $|S^l| = j$ 
%obtained by 
from 
partitioning a space of dimension $d$ with $n_l$ hyperplanes. We bound this value next.

For each $j$, there are at most $\binom{n_l}{j}$ regions with $|S^l| = j$, as they can be viewed as subsets of $n_l$ neurons of size $j$. In total, Lemma~\ref{lem:regions_rank_network} states that there are at most $\sum_{j=0}^{\min\{n_l, d\}} \binom{n_l}{j}$ regions. If we allow these regions to have the highest $|S^l|$ possible, for each $j$ from $0$ to $\min\{n_l, d\}$ we have at most $\binom{n_l}{n_l - j} = \binom{n_l}{j}$ regions with $|S^l| = n_l - j$.

Therefore, we can write the recurrence as
\begin{align*}
R(l, d) =
\begin{cases}
\displaystyle\sum_{j = 0}^{\min\{n_l, d\}} \binom{n_l}{j} R(l + 1, \min\{n_l - j, d\})\\
\hspace{10em}\text{if $1 \leq l \leq L - 1$,}\\
\displaystyle\sum_{j = 0}^{\min\{n_L, d\}} \binom{n_L}{j}\hspace{7em}\text{if $l = L$.}
\end{cases}
\end{align*}

The recurrence $R(1, n_0)$ can be unpacked to
\begin{align*}
\sum_{j_1=0}^{\min\{n_1, d_1\}} \binom{n_1}{j_1} \sum_{j_2=0}^{\min\{n_2, d_2\}} \binom{n_2}{j_2} \cdots \sum_{j_L=0}^{\min\{n_L, d_L\}} \binom{n_L}{j_L}
\end{align*}
where $d_l = \min\{n_0, n_1 - j_1, \ldots, n_{l-1} - j_{l-1}\}$. This can be made more compact, resulting in the final expression.

The bound is tight when $L = 1$ since it becomes $\sum_{j=0}^{\min\{n_0, n_1\}} \binom{n_1}{j}$, which is the maximal number of regions of a single-layer network.
\end{proof}

As a side note, Theorem~\ref{thm:upper_bound_improved} can be further tightened if the weight matrices are known to have small rank. The bound from Lemma~\ref{lem:regions_rank_network} can be rewritten as $\sum_{j=0}^{\min\{\operatorname{rank}(W^l), \operatorname{dim}(\mathbf{h}^{l-1}(\mathcal{S}))\}} \binom{n_l}{j}$ if we do not relax $\operatorname{rank}(W^l)$ to $n_l$ in the proof. The term $\operatorname{rank}(W^l)$ follows through the proof of Theorem~\ref{thm:upper_bound_improved} and the index set $J$ in the theorem becomes $\{(j_1, \ldots, j_L) \in \mathbb{Z}^L: 0 \leq j_l \leq \min\{n_0, n_1 - j_1, \ldots, n_{l-1} - j_{l-1}, \operatorname{rank}(W^l)\}\ \forall l \geq 1\}$.

A key insight from Lemmas~\ref{lem:regions_rank_network} and~\ref{lem:dim_bound} is that the dimensions of the images of the regions are non-increasing as we move through the layers partitioning them. In other words, if at any layer the dimension of the image of a region becomes small, then that region will not be able to be further partitioned into a large number of regions. For instance, if the dimension of the image of a region falls to zero, then that region will never be further partitioned. This suggests that if we want to have many regions, we need to keep dimensions high. We use this idea in the next section to construct a DNN with many regions.

\subsection{The Case of Dimension One}\label{sec:1D_analysis}

If the input dimension $n_0$ is equal to 1 and $n_l = n$ for all layers $l$, the upper bound presented in the previous section reduces to $(n+1)^L$. On the other hand, the lower bound given by~\citet{Montufar2014} becomes $n^{L-1}(n+1)$. It is then natural to ask: are either of these bounds tight? The answer is that the upper bound is tight in the case of $n_0 = 1$, assuming there are sufficiently many neurons.

\begin{thm}\label{thm:dimension_one}
Consider a deep rectifier network with $L$ layers, $n_l \geq 3$ rectified linear units at each layer $l$, and an input of dimension 1. The maximal number of regions of this neural network is exactly $\prod_{l=1}^L (n_l + 1)$.
\end{thm}

The expression above is a simplified form of the upper bound from Theorem~\ref{thm:upper_bound_improved} in the case $n_0 = 1$.

The proof of this theorem in Appendix~\ref{sec:proof_dim_one} has a construction with $n+1$ regions that replicate themselves as we add layers, instead of $n$ as in~\citet{Montufar2014}. That is motivated by an insight from the previous section: in order to obtain more regions, we want the dimension of the image of every region to be as large as possible. When $n_0 = 1$, we want all regions to have images with dimension one. This intuition leads to a new construction with one additional region that can be replicated with other strategies.

\subsection{Lower Bounds on the Maximal Number of Linear Regions}\label{sec:lower_bound_rectifier}

Both the lower bounds from~\citet{Montufar2014} and from~\citet{Arora2018} can be slightly improved, since their approaches are based on extending a 1-dimensional construction similar to the one in Section~\ref{sec:1D_analysis}. We do both since they are not directly comparable: the former bound is in terms of the number of neurons in each layer and the latter is in terms of the total size of the network.

\begin{thm}\label{thm:lower_bound}
The maximal number of linear regions induced by a rectifier network with $n_0$ input units and $L$ hidden layers with $n_l \geq 3n_0$ for all $l$ is lower bounded by $\big(\prod_{l=1}^{L-1}\big(\big\lfloor \frac{n_l}{n_0} \big\rfloor + 1\big)^{n_0}\big) \sum_{j=0}^{n_0} \binom{n_L}{j}$.
\end{thm}

The proof of this theorem is in Appendix~\ref{sec:proof_lower_bound}. For comparison, the differences between the lower bound theorem (Theorem 5) from \citet{Montufar2014} and the above theorem is the replacement of the condition $n_l \geq n_0$ by the more restrictive $n_l \geq 3n_0$, and of $\left\lfloor n_l/n_0 \right\rfloor$ by $\left\lfloor n_l/n_0 \right\rfloor + 1$.

\begin{thm}\label{thm:lower_bound_arora}
For any $m \geq 1$ and $w \geq 2$, there exists a rectifier network with $n_0$ input units and $L$ hidden layers of size $2m + w(L-1)$ that has $2\sum_{j=0}^{n_0-1}\binom{m-1}{j}(w+1)^{L-1}$ linear regions.
\end{thm}

The proof of this theorem is in Appendix~\ref{sec:proof_lower_bound_arora}. The differences between Theorem 2.11(i) from \citet{Arora2018} and the above theorem is the replacement of $w$ by $w+1$. They construct a $2m$-width layer with many regions and use a one-dimensional construction for the remaining layers.

\section{An Upper Bound on the Number of Linear Regions for Maxout Networks} \label{sec:bound_maxouts}

We now consider a deep neural network composed of maxout units. Given weights $W^l_j$ for $j = 1, \ldots, k$, the output of a rank-$k$ maxout layer $l$ is given by
\begin{eqnarray*}
\mathbf{h}^l & = &  \max\{W^l_1 \mathbf{h}^{l-1} + b^l_1, \ldots, W^l_k \mathbf{h}^{l-1} + b^l_k\}
\end{eqnarray*}

In terms of bounding number of regions, a major difference between the next result for maxout units and the previous one for ReLUs is that dimensionality plays less of a role, since neurons may no longer have an inactive state with zero output. Nevertheless, using techniques similar to the ones from Section~\ref{sec:upper_bound_rectifier}, the following theorem can be shown (see Appendix~\ref{sec:proof_upper_bound_maxout} for the proof).

\begin{thm}\label{thm:upper_bound_maxout}
Consider a deep neural network with $L$ layers, $n_l$ rank-$k$ maxout units at each layer $l$, and an input of dimension $n_0$. The maximal number of regions %of this neural network 
is at most
\[
\prod_{l=1}^L \sum_{j=0}^{d_l} \binom{\frac{k(k-1)}{2} n_l}{j}
\]
where $d_l = \min\{n_0, n_1, \ldots, n_l\}$.

Asymptotically, if $n_l = n$ for all $l = 1, \ldots, L$, $n \geq n_0$, and $n_0 = O(1)$, then the maximal number of regions is at most $O((k^2n)^{Ln_0})$.
\end{thm}

\section{Exact Counting of Linear Regions}\label{sec:counting}

If the input space $\mathbf{x} \in \mathbb{R}^{n_0}$ is bounded by minimum and maximum values along each dimension, or else if $\mathbf{x}$ corresponds to a polytope more generally, then we can define a mixed-integer linear formulation mapping polyhedral regions of $\mathbf{x}$ to the output space $\mathbf{y}$. %\in \mathbb{R}^{m}$. 
The assumption that $\mathbf{x}$ is bounded and polyhedral 
is natural in most applications, where each value $x_i$ has known lower and upper bounds (e.g., the value can vary from 0 to 1 for image pixels). Among other things, we can use this formulation to count the number of linear regions. 
 
In the formulation that follows, we use continuous variables to represent the input $\mathbf{x}$, which we can also denote as $\mathbf{h}^0$, the output of each neuron $i$ in layer $l$ as $h^l_i$, and the output $\mathbf{y}$ as $\mathbf{h}^{L+1}$. To simplify the representation, we lift this formulation to a space that also contains the output of a complementary set of neurons, each of which is active when the corresponding neuron is not. Namely, for each neuron $i$ in layer $l$ we also have a variable $\overline{h}^l_i := \max \{0, - W^l_i h^{l-1} - b_i^l\}$. We use binary variables of the form $z^l_i$ to denote if each neuron $i$ in layer $l$ is active or else if the complement %of such neuron is. 
is. Finally, we assume $M$ to be a sufficiently large constant. 

 For a given neuron $i$ in layer $l$, the following set of constraints maps the input to the output: 
\begin{align}
W^l_i h^{l-1} + b^l_i = h^l_i - \overline{h}^l_i \label{cons:hyp} \\ 
h^l_i \leq M z^l_i \label{cons:hu} \\
\overline{h}^l_i \leq M (1-z^l_i) \label{cons:su} \\
h^l_i \geq 0 \label{cons:hp} \\
\overline{h}^l_i \geq 0 \label{cons:sp} \\
z^l_i \in \{0, 1\} \label{cons:zb}
\end{align}

\begin{thm}\label{thm:mir}
Provided that $| W_i^l h^{l-1} + b_i^l | \leq M$ for any possible value of $h^{l-1}$, a formulation with the set of constraints \eqref{cons:hyp}--\eqref{cons:zb} for each neuron of a rectifier network is such that a feasible solution with a fixed value for $x$ yields the output $y$ of the neural network.
\end{thm}
\begin{proof}
It suffices to prove that the constraints for each neuron map the input to the output in the same way that the neural network would. 
If $W_i^l \mathbf{h}^{l-1} + b_i^l > 0$, it follows that $h_i^l - \overline{h}_i^l > 0$ according to \eqref{cons:hyp}. Since both variables are non-negative due to \eqref{cons:hp} and \eqref{cons:sp} whereas one is non-positive due to \eqref{cons:hu}, \eqref{cons:su}, and \eqref{cons:zb}, 
then $z_i^l = 1$ and $h_i^l = \max\left\{0, W_i^l \mathbf{h}^{l-1} + b_i^l \right\}$. 
If $W_i^l \mathbf{h}^{l-1} + b_i^l < 0$, 
then it similarly follows that $h_i^l - \overline{h}_i^l < 0$, $z_i^l = 0$, 
and thus $\overline{h}_i^l = \min\left\{0, W_i^l \mathbf{h}^{l-1} + b_i^l \right\}$.  If $W_i^l \mathbf{h}^{l-1} + b_i^l = 0$, then either $h_i^l = 0$ or $\overline{h}_i^l = 0$ due to constraints \eqref{cons:hp} to \eqref{cons:zb} whereas \eqref{cons:hyp} implies that $\overline{h}_i^l = 0$ or $h_i^l = 0$, respectively. In this case, the value of $z_i^l$ is arbitrary 
but irrelevant.
\end{proof}

A systematic method to count integer solutions is the one-tree approach~\citep{Danna1}, 
 which resumes the search after an optimal solution has been found using the same branch-and-bound tree. This method also allows for counting feasible solutions within a threshold of the optimal value. Note that in constraints \eqref{cons:hyp}--\eqref{cons:zb}, the variables $z^l_i$ can be either 0 or 1 when they lie on the activation boundary, whereas we want to consider a neuron active only when its output is strictly positive. This discrepancy may cause double-counting when activation boundaries overlap.
We can address that by defining an objective function that maximizes the minimum output $f$ of an active neuron, which is positive in the non-degenerate cases that we want to count. 
We state this formulation for rectifier networks as follows:
\begin{align*}
\max ~ ~ & f \nonumber \\
\text{s.t.} ~ ~ & \text{\eqref{cons:hyp}--\eqref{cons:zb}} & \forall \text{ neuron $i$ in layer $l$} \tag{$\mathcal{P}$}\label{mip} \\
& f \leq h^l_i + (1-z_i^l) M & \forall \text{ neuron $i$ in layer $l$} \nonumber \\
& x \in X \nonumber
\end{align*}

A discusssion on the choice for the $M$ constants can be found in Appendix~\ref{sec:mir_impl}.
 In addition, we discuss the mixed-representability of DNNs in Appendix~\ref{sec:mir}, the theory for unrestricted inputs in Appendix~\ref{sec:unrestricted_inputs}, and a mixed-integer formulation for maxout networks in Appendix~\ref{sec:mir_maxouts}, respectively.

%!TEX root = bounding_counting.tex

\section{Experiments}\label{sec:exp}

We perform an experiment to count linear regions of small-sized networks with ReLU activation units on the MNIST benchmark dataset~\citep{LeCun1998}. In this experiment, we train rectifier networks with two hidden layers summing up to 22 neurons.  We train 10 networks for each configuration for 20 epochs or training steps, and we count all linear regions within $0 \leq \mathbf{x} \leq 1$. 
The counting code is written in C++ (gcc 4.8.4) using CPLEX Studio 12.8 as a solver and ran in Ubuntu 14.04.4 on a machine with 40 Intel(R) Xeon(R)
CPU E5-2640 v4 @ 2.40GHz processors and 132 GB of RAM. 
The runtimes for counting different configuration can be found in Appendix~\ref{sec:runtimes}.

 Figure~\ref{fig:bounds_counting} shows the average, minimum, and maximum number of linear regions for each configuration of these networks. The plot also contains the corresponding upper bounds for each configuration from Theorem~\ref{thm:upper_bound_improved} and those from \citet{Montufar2014} and \citet{Montufar2017}, 
which are the first and the tightest bounds from prior work respectively. Note that upper bound from Theorem~\ref{thm:upper_bound_improved} is tighter. %compared to the previous results (\citet{Montufar2014} and \citet{Montufar2017}). 

Figure~\ref{fig:all_error} shows how the number of regions for each trained DNN compares with Cross-Entropy~(CE) error on training set, and Misclassification Rate~(MR) on testing set. If an intermediate layer in a network has only 1 or 2 neurons, it is natural to expect very high training and test errors.
We discard such DNNs with a layer of small width in Figure~\ref{fig:wider_error}.

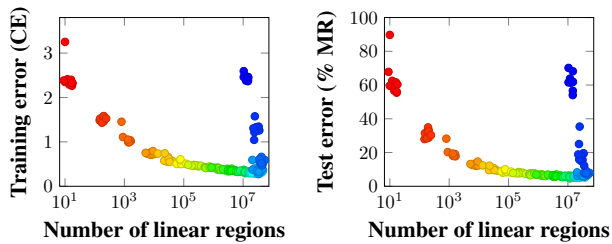
\begin{figure}%[htbp]
  \centering
\begin{tikzpicture}[scale=0.71]
\begin{semilogxaxis}[filter discard warning=false, enlargelimits=false, enlarge x limits=0.02, 
title style={at={(0.5,-0.25)},anchor=north,yshift=-0.1},
xlabel={\large \textbf{Number of linear regions}},
ylabel={\large \textbf{Training error (CE)}},ymin=0.0,ymax=3.8, 
every axis y label/.style={at={(ticklabel cs:0.5)},rotate=90,anchor=near ticklabel}, 
legend style={at={(0.55,0.35)},anchor=north},
width=5.5cm,
colormap={my colormap}{
                color=(green)
                color=(red)
                color=(blue)
                color=(yellow)
            }
]
\addplot[
        scatter,%
        scatter/@pre marker code/.code={%
            \edef\temp{\noexpand\definecolor{mapped color}{rgb}{\pgfplotspointmeta}}%
            \temp
            \scope[draw=mapped color!80!black,fill=mapped color]%
        },%
        scatter/@post marker code/.code={%
            \endscope
        },%
        only marks,     
        mark=*,
        point meta={TeX code symbolic={%
            \edef\pgfplotspointmeta{\thisrow{RED},\thisrow{GREEN},\thisrow{BLUE}}%
        }},
    ] 
 table {plot_experiments_A-training.txt} ;
\end{semilogxaxis}
\end{tikzpicture}
~ 
\begin{tikzpicture}[scale=0.71]
\begin{semilogxaxis}[filter discard warning=false, enlargelimits=false, enlarge x limits=0.02, 
title style={at={(0.5,-0.25)},anchor=north,yshift=-0.1},
xlabel={\large \textbf{Number of linear regions}},
ylabel={\large \textbf{Test error (\% MR)}},ymin=0.0,ymax=100.0, 
every axis y label/.style={at={(ticklabel cs:0.5)},rotate=90,anchor=near ticklabel}, 
legend style={at={(0.55,0.35)},anchor=north},
width=5.5cm,
colormap={my colormap}{
                color=(green)
                color=(red)
                color=(blue)
                color=(yellow)
            }
]
\addplot[
        scatter,%
        scatter/@pre marker code/.code={%
            \edef\temp{\noexpand\definecolor{mapped color}{rgb}{\pgfplotspointmeta}}%
            \temp
            \scope[draw=mapped color!80!black,fill=mapped color]%
        },%
        scatter/@post marker code/.code={%
            \endscope
        },%
        only marks,     
        mark=*,
        point meta={TeX code symbolic={%
            \edef\pgfplotspointmeta{\thisrow{RED},\thisrow{GREEN},\thisrow{BLUE}}%
        }},
    ] 
table {plot_experiments_B-test.txt} ;
\end{semilogxaxis}
\end{tikzpicture}
\caption{\emph{Number of linear regions versus cross-entropy error on the training set, and misclassification rate on the testing set. The colors represent the number of neurons in each of the first two layers. The color gradient from red to blue corresponds to increasing the size of the first layer and decreasing the size of the second.}}
\label{fig:all_error}
\end{figure}

\begin{figure}%[htbp]
  \centering
\begin{tikzpicture}[scale=0.71]
\begin{semilogxaxis}[filter discard warning=false, enlargelimits=false, enlarge x limits=0.02, 
title style={at={(0.5,-0.25)},anchor=north,yshift=-0.1},
xlabel={\large \textbf{Number of linear regions}},
ylabel={\large \textbf{Training error (CE)}},ymin=0.2,ymax=0.9, 
every axis y label/.style={at={(ticklabel cs:0.5)},rotate=90,anchor=near ticklabel}, 
legend style={at={(0.55,0.35)},anchor=north},
width=5.5cm,
colormap={my colormap}{
                color=(green)
                color=(red)
                color=(blue)
                color=(yellow)
            }
]
\addplot[
        scatter,%
        scatter/@pre marker code/.code={%
            \edef\temp{\noexpand\definecolor{mapped color}{rgb}{\pgfplotspointmeta}}%
            \temp
            \scope[draw=mapped color!80!black,fill=mapped color]%
        },%
        scatter/@post marker code/.code={%
            \endscope
        },%
        only marks,     
        mark=*,
        point meta={TeX code symbolic={%
            \edef\pgfplotspointmeta{\thisrow{RED},\thisrow{GREEN},\thisrow{BLUE}}%
        }},
    ] 
 table {plot_experiments_A-training-R.txt} ;
\end{semilogxaxis}
\end{tikzpicture}
~ 
\begin{tikzpicture}[scale=0.71]
\begin{semilogxaxis}[filter discard warning=false, enlargelimits=false, enlarge x limits=0.02, 
title style={at={(0.5,-0.25)},anchor=north,yshift=-0.1},
xlabel={\large \textbf{Number of linear regions}},
ylabel={\large \textbf{Test error (\%MR)}},ymin=4.0,ymax=15.0, 
every axis y label/.style={at={(ticklabel cs:0.5)},rotate=90,anchor=near ticklabel}, 
legend style={at={(0.55,0.35)},anchor=north},
width=5.5cm,
colormap={my colormap}{
                color=(green)
                color=(red)
                color=(blue)
                color=(yellow)
            }
]
\addplot[
        scatter,%
        scatter/@pre marker code/.code={%
            \edef\temp{\noexpand\definecolor{mapped color}{rgb}{\pgfplotspointmeta}}%
            \temp
            \scope[draw=mapped color!80!black,fill=mapped color]%
        },%
        scatter/@post marker code/.code={%
            \endscope
        },%
        only marks,     
        mark=*,
        point meta={TeX code symbolic={%
            \edef\pgfplotspointmeta{\thisrow{RED},\thisrow{GREEN},\thisrow{BLUE}}%
        }},
    ] 
table {plot_experiments_B-test-R.txt} ;
\end{semilogxaxis}
\end{tikzpicture}
\caption{\emph{Scatter plot for the number of linear regions versus training and testing errors on DNNs with minimum width at least 4. The colors follow the same convention used in Figure~\ref{fig:all_error}.}}
\label{fig:wider_error}
\end{figure}
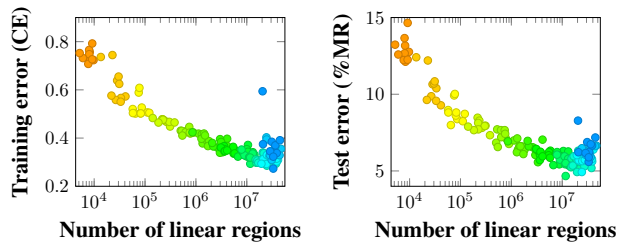
\section{Discussion}

Our ReLU upper bound indicates that small widths in early layers cause a bottleneck effect on the number of regions. If we reduce the width of an early layer, the dimension of the image of the linear regions become irrecoverably smaller throughout the network and the regions will not be able to be partitioned as much. This intuition allows us to develop a 1-dimensional construction with the maximal number of regions by eliminating a zero-dimensional bottleneck. In our experiment, we validate the bottleneck effect by observing that the actual number of linear regions is asymmetric when one layer is increased and another decreased in size.

An unexpected consequence of one of our results is that shallow networks can attain more linear regions when the input dimensions exceed the number of neurons. % of the {\sc DNN}.
This complements prior work, which has not considered large input dimensions, a common case in practice.

We also observe in Figure~\ref{fig:insightsB} that the depth that maximizes the upper bound from Theorem~\ref{thm:upper_bound_improved} increases with the number of units and decreases with the size of the input. It would be interesting to investigate if this also happens with respect to the actual number of regions, in which case the depth could be chosen according to those parameters.

However, it would be possible that DNNs configurations with large number of linear regions do not generalize
well if there are so many regions that each training point can be singled out in a different region, in particular if regions with similar labels are unlikely to be compositionally related.

Nevertheless, we have initial evidence that training and classification accuracies relate to the number of regions for configurations where no layer is too small, hence suggesting that the number of regions can be a metric for comparing similar DNN configurations. 
However, in the cases where one layer is too small, we have also observed that the number of regions can be large and do not reflect the capacity of the DNN. We hypothesize that the presence of low dimensionality negatively affects the accuracy of the network, 
even when the number of regions is large because of other layers. 
This indicates that potentially more insights could be gained from investigating the shape of linear regions.

\section*{Acknowledgements}
We thank the anonymous reviewers for useful suggestions. 

\bibliography{icml2018_conference}
\bibliographystyle{icml2018}
\newpage
\onecolumn
%!TEX root = LinearRegionsMIR.tex
\appendix
\appendixpage
Most of the proofs for theorems and lemmas associated with the upper and lower bounds on the linear regions are provided below. We also discuss counting on unrestricted input and the mixed-integer formulation for maxout networks.

\section{Analysis of the Bound from Theorem~\ref{thm:upper_bound_improved}}\label{sec:analysis_bound}

In this section, we present properties of the upper bound for the number of regions of a rectifier network from Theorem~\ref{thm:upper_bound_improved}. Denote the bound by $B(n_0, n_1, \ldots, n_L)$, where $n_0$ is the input dimension and $n_1, \ldots, n_L$ are the widths of layers 1 through $L$ of the network. That is,
\begin{align*}
B(n_0, n_1, \ldots, n_L) := \sum_{(j_1,\ldots,j_L) \in J} \prod_{l=1}^L \binom{n_l}{j_l}
\end{align*}

Instead of expressing $J$ as in Theorem~\ref{thm:upper_bound_improved}, we rearrange it to a more convenient form for the proofs in this section:
\begin{align*}
J = \{(j_1, \ldots, j_L) \in \mathbb{Z}^L : & j_l + j_k \leq n_k\;\forall k = 1, \ldots, l-1\;\forall l = 2, \ldots, L\\
& j_l \leq n_0\;\forall l = 1, \ldots, L\\
& 0 \leq j_l \leq n_l\;\forall l = 1, \ldots, L\}.
\end{align*}

Note that whenever we assume $n_0 \geq \max\{n_1, \ldots, n_l\}$, then the bound inequality for $n_0$ becomes redundant and can be removed.

Some of the results have implications in terms of the exact maximal number of regions. We denote it by $R(n_0, n_1, \ldots, n_L)$, following the same notation above.

Moreover, the following lemma is useful throughout the section.
\begin{lem}\label{lem:vandermonde}
\begin{align*}
\sum_{j = 0}^k \binom{n_1 + \ldots + n_L}{j} = \sum_{\substack{j_1 + \ldots + j_L \leq k\\0 \leq j_l \leq n_l\;\forall l}} \binom{n_1}{j_1} \binom{n_2}{j_2} \cdots \binom{n_L}{j_L}.
\end{align*}
\end{lem}
\begin{proof}
The result comes from taking a generalization of Vandermonde's identity and adding the summation of $j$ from $0$ to $k$ as above.
\end{proof}

We first examine some properties related to 2-layer networks. The proposition below characterizes the bound when $L = 2$ for large input dimensions.

\begin{prp}\label{prp:two_layer_bound}
Consider a 2-layer network with widths $n_1, n_2$ and input dimension $n_0 \geq n_1$ and $n_0 \geq n_2$. Then
\begin{align*}
B(n_0, n_1, n_2) = \sum_{j=0}^{n_1} \binom{n_1 + n_2}{j}
\end{align*}
If $n_0 < n_1$ or $n_0 < n_2$, the above holds with inequality: $B(n_0, n_1, n_2) \leq \sum_{j=0}^{n_1} \binom{n_1 + n_2}{j}$.
\end{prp}
\begin{proof}
If $n_0 \geq n_1$ and $n_0 \geq n_2$, the bound inequalities for $n_0$ in the index set $J$ become redundant. By applying Lemma~\ref{lem:vandermonde}, we obtain
\begin{align*}
B(n_0, n_1, n_2) = \sum_{\substack{0 \leq j_1 + j_2 \leq n_1\\0 \leq j_l \leq n_l\ \forall l}} \binom{n_1}{j_1} \binom{n_2}{j_2} = \sum_{j=0}^{n_1} \binom{n_1 + n_2}{j}.
\end{align*}

If $n_0 < n_1$ or $n_0 < n_2$, then its index set $J$ is contained by the one above, and thus the first equal sign above becomes a less-or-equal sign.
\end{proof}

Recall that the expression on the right-hand side of Proposition~\ref{prp:two_layer_bound} is equal to the maximal number of regions of a single-layer network with $n_1 + n_2$ ReLUs and input dimension $n_1$, as discussed in Section~\ref{sec:background}. Hence, the proposition implies that for large input dimensions, a two-layer network has no more regions than a single-layer network with the same number of neurons, as formalized below.

\begin{cor}
Consider a 2-layer network with widths $n_1, n_2 \geq 1$ and input dimension $n_0 \geq n_1$ and $n_0 \geq n_2$. Then $R(n_0, n_1, n_2) \leq R(n_0, n_1 + n_2)$.

Moreover, this inequality is strict when $n_0 > n_1$.
\end{cor}
\begin{proof}
This is a direct consequence of Proposition~\ref{prp:two_layer_bound}:
\begin{align*}
R(n_0, n_1, n_2) \leq B(n_0, n_1, n_2) = \sum_{j=0}^{n_1} \binom{n_1 + n_2}{j} \leq \sum_{j=0}^{n_0} \binom{n_1 + n_2}{j} = R(n_0, n_1 + n_2).
\end{align*}
Note that if $n_0 > n_1$, then the second inequality can be turned into a strict inequality.
\end{proof}

The next proposition illustrates the bottleneck effect for two layers. It states that for large input dimensions, moving a neuron from the second layer to the first strictly increases the bound. This proposition is stated more informally in the main text.

\begin{refprp}{prp:bottleneck_2_layer}
Consider a 2-layer network with widths $n_1, n_2$ and input dimension $n_0 \geq n_1 + 1$ and $n_0 \geq n_2 + 1$. Then $B(n_0, n_1 + 1, n_2) > B(n_0, n_1, n_2 + 1)$.
\end{refprp}
\begin{proof}
By Proposition~\ref{prp:two_layer_bound},
\begin{align*}
B(n_0, n_1 + 1, n_2) = \sum_{j=0}^{n_1 + 1} \binom{(n_1 + 1) + n_2}{j} > \sum_{j=0}^{n_1} \binom{n_1 + (n_2 + 1)}{j} = B(n_0, n_1, n_2 + 1).
\end{align*}
\end{proof}

The assumption that $n_0$ must be large is required for the above proposition; otherwise, the input itself may create a bottleneck with respect to the second layer as we decrease its size. Note that the bottleneck affects all subsequent layers, not only the layer immediately after it.

However, it is not true that moving neurons to earlier layers always increases the bound. For instance, with three layers, $B(4, 3, 2, 1) = 47 > 46 = B(4, 4, 1, 1)$. 

In the remainder of this section, we consider deep networks of equal widths $n$. The next proposition can be viewed as an extension of Proposition~\ref{prp:two_layer_bound} for multiple layers. It states that for a network with widths and input dimension $n$ and at least 4 layers, if we halve the number of layers and redistribute the neurons so that the widths become $2n$, then the bound increases. In other words, if we assume the bound to be close to the maximal number of regions, it suggests that making a deep network shallower allows for more regions when the input dimension is equal to the width.

\begin{prp}
Consider a $2L$-layer network with equal widths $n$ and input dimension $n_0 = n$. Then
\begin{align*}
B(n, \underbrace{n, \ldots, n}_{\text{$2L$ times}}) \leq B(n, \underbrace{2n, \ldots, 2n}_{\text{$L$ times}}).
\end{align*}
This inequality is met with equality when $L = 1$ and strict inequality when $L \geq 2$.
\end{prp}
\begin{proof}
When $n_0 = n$, the inequalities $j_l \leq \min\{n_0, 2n - j_1, \ldots, 2n - j_{l-1}, 2n\}$ appearing in $J$ (in the form presented in Theorem~\ref{thm:upper_bound_improved}) can be simplified to $j_l \leq n$. Therefore, using Lemma~\ref{lem:vandermonde}, the bound on the right-hand side becomes
\begin{align*}
B(n, \underbrace{2n, \ldots, 2n}_{\text{$L$ times}}) &= \sum_{j_1 = 0}^n \sum_{j_2 = 0}^n \ldots \sum_{j_L = 0}^n \prod_{l=1}^L \binom{2n}{j_l} = \left(\sum_{j = 0}^n \binom{2n}{j}\right)^L = \left(\sum_{j_1 = 0}^n \sum_{j_2 = 0}^{n - j_1} \binom{n}{j_1} \binom{n}{j_2} \right)^L\\
& \geq \sum_{(j_1,\ldots,j_{2L}) \in J} \prod_{l=1}^{2L} \binom{n}{j_l} = B(n, \underbrace{n, \ldots, n}_{\text{$2L$ times}}).
\end{align*}
where $J$ above is the index set from Theorem~\ref{thm:upper_bound_improved} applied to $n_0 = n_l = n$ for all $l = 1, \ldots, 2L$. Note that we can turn the inequality into equality when $L = 1$ (also becoming a consequence of Proposition~\ref{prp:two_layer_bound}) and into strict inequality when $L \geq 2$.
\end{proof}

Next, we provide an upper bound that is independent of $n_0$ based on Theorem~\ref{thm:upper_bound_improved}.

\begin{cor}\label{cor:bound_to_bound}
Consider an $L$-layer network with equal widths $n$ and any input dimension $n_0 \geq 0$.
\begin{align*}
B(n_0, n, \ldots, n) \leq 2^{Ln} \left(\frac{1}{2} + \frac{1}{2 \sqrt{\pi n}}\right)^{L/2} \sqrt{2}
\end{align*}
\end{cor}
\begin{proof}
Since we are deriving an upper bound, we can assume $n_0 \geq n$, as the bound is nondecreasing on $n_0$. We first assume that $L$ is even. We relax some of the constraints of the index set $J$ from Theorem~\ref{thm:upper_bound_improved} and apply Vandermonde's identity on each pair:
\begin{align*}
B(n_0, n, \ldots, n) &\leq \sum_{j_1 = 0}^n \sum_{j_2 = 0}^{n-j_1} \binom{n}{j_1} \binom{n}{j_2} \sum_{j_3 = 0}^n \sum_{j_4 = 0}^{n-j_3} \binom{n}{j_3} \binom{n}{j_4} \cdots \sum_{j_{L-1} = 0}^n \sum_{j_L = 0}^{n-j_{L-1}} \binom{n}{j_{L-1}} \binom{n}{j_L}\\
&= \left(\sum_{j=0}^n \binom{2n}{j}\right)^{L/2} = \left(\frac{2^{2n} + \binom{2n}{n}}{2}\right)^{L/2} \leq \left(\frac{2^{2n} + \frac{2^{2n}}{\sqrt{\pi n}}}{2}\right)^{L/2}\\
&= 2^{Ln} \left(\frac{1}{2} + \frac{1}{2\sqrt{\pi n}}\right)^{L/2}.
\end{align*}

The bound on $\binom{2n}{n}$ is a direct application of Stirling's approximation~\citep{Stirling}. If $L$ is odd, then we can write
\begin{align*}
B(n_0, n, \ldots, n) &\leq \left(\sum_{j=0}^n \binom{2n}{j}\right)^{(L-1)/2} \left(\sum_{j=0}^n \binom{n}{j}\right) = \left(\sum_{j=0}^n \binom{2n}{j}\right)^{L/2} \frac{2^n}{\left(\sum_{j=0}^n \binom{2n}{j}\right)^{1/2}}\\
& \leq \left(\sum_{j=0}^n \binom{2n}{j}\right)^{L/2} \frac{2^n}{(2^{2n} / 2)^{1/2}} = \left(\sum_{j=0}^n \binom{2n}{j}\right)^{L/2} \sqrt{2}\\
& \leq 2^{Ln} \left(\frac{1}{2} + \frac{1}{2\sqrt{\pi n}}\right)^{L/2} \sqrt{2}
\end{align*}
where the last inequality is analogous to the even case. Hence, the result follows.
\end{proof}

We now use Corollary~\ref{cor:bound_to_bound} to show that a shallow network can attain more linear regions than a deeper one of the same size when the input dimension $n_0$ exceeds the total number of neurons. This result is stated in the main text.

\begin{refcor}{cor:shallow_vs_deep}
Let $L \geq 2$, $n \geq 1$, and $n_0 \geq Ln$. Then
\begin{align*}
R(n_0, \underbrace{n, \ldots, n}_{\text{$L$ times}}) < R(n_0, Ln)
\end{align*}

Moreover, $\lim_{L \to \infty} \frac{R(n_0, n, \ldots, n)}{R(n_0, Ln)} = 0$.
\end{refcor}
\begin{proof}
Note first that if $n_0 \geq Ln$, then $R(n_0, Ln) = \sum_{j=0}^{Ln} \binom{Ln}{j} = 2^{Ln}$. 

The inequality can be derived from Proposition~\ref{cor:bound_to_bound}, since $\left(\frac{1}{2} + \frac{1}{2 \sqrt{\pi n}}\right)^{L/2} \sqrt{2} < 1$ when $L \geq 3$ and $n \geq 1$, and thus $R(n_0, n, \ldots, n) \leq B(n_0, n, \ldots, n) < 2^{Ln} = R(n_0, Ln)$. The same holds when $L = 2$: as noted in the proof of Proposition~\ref{cor:bound_to_bound}, we may discard the factor of $\sqrt{2}$ when $L$ is even, and $\left(\frac{1}{2} + \frac{1}{2 \sqrt{\pi n}}\right)^{L/2} < 1$ for $L = 2$ and $n \geq 1$.

Proposition~\ref{cor:bound_to_bound} also implies that
\begin{align*}
\frac{R(n_0, n, \ldots, n)}{R(n_0, Ln)} \leq \frac{B(n_0, n, \ldots, n)}{2^{Ln}} \leq \left(\frac{1}{2} + \frac{1}{2\sqrt{\pi n}}\right)^{L/2} \sqrt{2}.
\end{align*}

Since the base of the first term of the above expression is less than 1 for all $n \geq 1$ and $\sqrt{2}$ is a constant, the ratio goes to 0 as $L$ goes to infinity.
\end{proof}

\section{Exponential Maximal Number of Regions When Input Dimension is Large}\label{sec:exp_number_regions}

\begin{prp}
Consider an $L$-layer rectifier network with equal widths $n$ and input dimension $n_0 \geq n / 3$. Then the maximal number of regions is $\Omega(2^{\frac{2}{3}Ln})$.
\end{prp}
\begin{proof}
It suffices to show that a lower bound such as the one from Theorem~\ref{thm:lower_bound} grows exponentially large. For simplicity, we consider the lower bound $(\prod_{l=1}^L (\lfloor n_l / n_0 \rfloor + 1))^{n_0}$, which is the bound obtained before the last tightening step in the proof of Theorem~\ref{thm:lower_bound} (see Appendix~\ref{sec:proof_lower_bound}).

Note that replacing $n_0$ in the above expression by a value $n_0'$ smaller than the input dimension still yields a valid lower bound. This holds because increasing the input dimension of a network from $n_0'$ to $n_0$ cannot decrease its maximal number of regions.

Choose $n_0' = \lfloor n / 3 \rfloor$, which satisfies $n_0' \leq n_0$ and the condition $n \geq 3n_0'$ of Theorem~\ref{thm:lower_bound}. The lower bound can be expressed as $( \lfloor n / \lfloor n/3 \rfloor \rfloor + 1)^{L \lfloor n/3 \rfloor} \geq 4^{L \lfloor n/3 \rfloor}$. This implies that the maximal number of regions is $\Omega(2^{\frac{2}{3}Ln})$.
\end{proof}

\section{Proof of Lemma~\ref{lem:regions_rank}}\label{sec:proof_regions_rank}

\begin{reflem}{lem:regions_rank}
Consider $m$ hyperplanes in $\mathbb{R}^d$ defined by the rows of $Wx + b = 0$. Then the number of regions induced by the hyperplanes is at most $\sum_{j=0}^{\operatorname{rank}(W)} \binom{m}{j}$.
\end{reflem}
\begin{proof}
Consider the row space $\mathcal{R}(W)$ of $W$, which is a subspace of $\mathbb{R}^d$ of dimension $\operatorname{rank}(W)$. We show that the number of regions $N_{\mathbb{R}^d}$ in $\mathbb{R}^d$ is equal to the number of regions $N_{\mathcal{R}(W)}$ in $\mathcal{R}(W)$ induced by $Wx + b = 0$ restricted to $\mathcal{R}(W)$. This suffices to prove the lemma since $\mathcal{R}(W)$ has at most $\sum_{j=0}^{\operatorname{rank}(W)} \binom{m}{j}$ regions according to Zaslavsky's theorem.

Since $\mathcal{R}(W)$ is a subspace of $\mathbb{R}^d$, it directly follows that $N_{\mathcal{R}(W)} \leq N_{\mathbb{R}^d}$. To show the converse, we apply the orthogonal decomposition theorem from linear algebra: any point $\bar{x} \in \mathbb{R}^d$ can be expressed uniquely as $\bar{x} = \hat{x} + y$, where $\hat{x} \in \mathcal{R}(W)$ and $y \in \mathcal{R}(W)^{\perp}$. Here, $\mathcal{R}(W)^{\perp} = \text{Ker}(W) := \{y \in \mathbb{R}^d: Wy = 0\}$, and thus $W\bar{x} = W\hat{x} + Wy = W\hat{x}$. This means $\bar{x}$ and $\hat{x}$ lie on the same side of each hyperplane of $Wx + b = 0$ and thus belong to the same region. In other words, given any $\bar{x} \in \mathbb{R}^d$, its region is the same one that $\hat{x} \in \mathcal{R}(W)$ lies in. Therefore, $N_{\mathbb{R}^d} \leq N_{\mathcal{R}(W)}$. Hence, $N_{\mathbb{R}^d} = N_{\mathcal{R}(W)}$ and the result follows.
\end{proof}

\section{Proof of Theorem~\ref{thm:dimension_one}}\label{sec:proof_dim_one}

\begin{refthm}{thm:dimension_one}
Consider a deep rectifier network with $L$ layers, $n_l \geq 3$ rectified linear units at each layer $l$, and an input of dimension 1. The maximal number of regions of this neural network is exactly $\prod_{l=1}^L (n_l + 1)$.
\end{refthm}
\begin{proof}
Section~\ref{sec:bound_rectifier} provides us with a helpful insight to construct an example with a large number of regions. It tells us that we want regions to have images with large dimension in general. In particular, regions with images of dimension zero cannot be further partitioned. This suggests that the one-dimensional construction from~\citet{Montufar2014} can be improved, as it contains $n$ regions with images of dimension one and 1 region with image of dimension zero. This is because all ReLUs point to the same direction as depicted in Figure~\ref{fig:1d_construction}, leaving one region with an empty activation pattern.

Our construction essentially increases the dimension of the image of this region from zero to one. This is done by shifting the neurons forward and flipping the direction of the third neuron, as illustrated in Figure~\ref{fig:1d_construction}. We assume $n \geq 3$.

\begin{figure}[htp]
  \centering
    \includegraphics[width=0.8\textwidth]{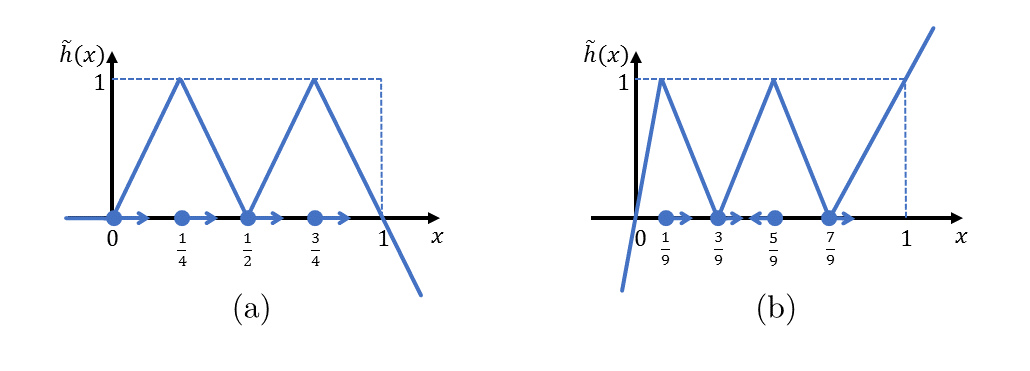}
  \caption{\emph{(a) The 1D construction from~\citet{Montufar2014}. All units point to the right, leaving a region with image of dimension zero before the origin. (b) The 1D construction described in this section. Within the interval $[0,1]$ there are five regions instead of the four in (a).}}
	\label{fig:1d_construction}
\end{figure}

We review the intuition behind the construction strategy from~\citet{Montufar2014}. They construct a linear function $\tilde{h} : \mathbb{R} \to \mathbb{R}$ with a zigzag pattern from $[0,1]$ to $[0,1]$ that is composed of $n$ ReLUs. More precisely, $\tilde{h}(x) = [1, -1, 1, \ldots, (-1)^{n+1}] [h_1(x), h_2(x), \ldots, h_n(x)]^\top$, where $h_i(x)$ for $i=1, \ldots, n$ are ReLUs. This linear function can be absorbed in the preactivation function of the next layer.

The zigzag pattern allows it to replicate in each slope a scaled copy of the function in the domain $[0,1]$. Figure~\ref{fig:1d_replication} shows an example of this effect. Essentially, when we compose $\tilde{h}$ with itself, each linear piece in $[t_1, t_2]$ such that $\tilde{h}(t_1) = 0$ and $\tilde{h}(t_2) = 1$ maps the entire function $\tilde{h}$ to the interval $[t_1, t_2]$, and each piece such that $\tilde{h}(t_1) = 1$ and $\tilde{h}(t_2) = 0$ does the same in a backward manner.

\begin{figure}[htp]
  \centering
    \includegraphics[width=\textwidth]{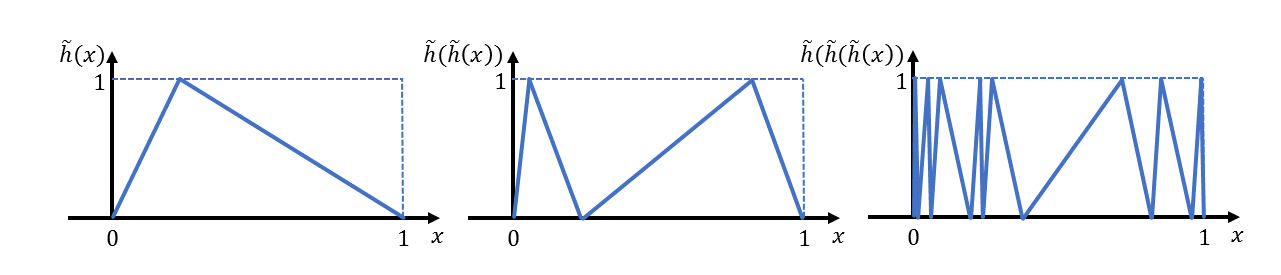}
  \caption{\emph{A function with a zigzag pattern composed with itself. Note that the entire function is replicated within each linear region, up to a scaling factor.}}
	\label{fig:1d_replication}
\end{figure}

In our construction, we want to use $n$ ReLUs to create $n+1$ regions instead of $n$. In other words, we want the construct this zigzag pattern with $n+1$ slopes. In order to do that, we take two steps to give ourselves more freedom. First, observe that we only need each linear piece to go from zero to one or one to zero; that is, the construction works independently of the length of each piece. Therefore, we turn the breakpoints into parameters $t_1, t_2, \ldots, t_n$, where $0 < t_1 < t_2 < \ldots < t_n < 1$. Second, we add sign and bias parameters to the function $\tilde{h}$. That is, $\tilde{h}(x) = [s_1, s_2, \ldots, s_n] [h_1(x), h_2(x), \ldots, h_n(x)]^\top + d$, where $s_i \in \{-1, +1\}$ and $d$ are parameters to be set. Here, $h_i(x) = \max \{0, \tilde{w}_i x + \tilde{b}_i\}$ since it is a ReLU.

We define $w_i = s_i \tilde{w}_i$ and $b_i = s_i \tilde{b}_i$, which are the weights and biases we seek in each interval to form the zigzag pattern. The parameters $s_i$ are needed because the signs of $\tilde{w}_i$ cannot be arbitrary: it must match the directions the ReLUs point towards. In particular, we need a positive slope ($\tilde{w}_i > 0$) if we want $i$ to point right, and a negative slope ($\tilde{w}_i < 0$) if we want $i$ to point left. Hence, without loss of generality, we do not need to consider the $s_i$'s any further since they will be directly defined from the signs of the $w_i$'s and the directions. More precisely, $s_i = 1$ if $w_i \geq 0$ and $s_i = -1$ otherwise for $i = 1, 2, 4, \ldots, n$, and $s_3 = -1$ if $w_3 \geq 0$ and $s_3 = 1$ otherwise.

To summarize, our parameters are the weights $w_i$ and biases $b_i$ for each ReLU, a global bias $d$, and the breakpoints $0 < t_1 < \ldots < t_n < 1$. Our goal is to find values for these parameters such that each piece in the function $\tilde{h}$ with domain in $[0,1]$ is linear from zero to one or one to zero.

More precisely, if the domain is $[s, t]$, we want each linear piece to be either $\frac{1}{t - s}x - \frac{s}{t - s}$ or $-\frac{1}{t - s}x + \frac{t}{t - s}$, which define linear functions from zero to one and from one to zero respectively. Since we want a zigzag pattern, the former should happen for the interval $[t_i, t_{i-1}]$ when $i$ is odd and the latter should happen when $i$ is even.

There is one more set of parameters that we will fix. Each ReLU corresponds to a hyperplane, or a point in dimension one. In fact, these points are the breakpoints $t_1, \ldots, t_n$. They have directions that define for which inputs the neuron is activated. For instance, if a neuron $h_i$ points to the right, then the neuron $h_i(x)$ outputs zero if $x \leq t_i$ and the linear function $w_i x + b_i$ if $x > t_i$.

As previously discussed, in our construction all neurons point right except for the third neuron $h_3$, which points left. This is to ensure that the region before $t_1$ has one activated neuron instead of zero, which would happen if all neurons pointed left. However, although ensuring every region has images of dimension one is necessary to reach the bound, not every set of directions yields valid weights. These directions are chosen so that they admit valid weights.

The directions of the neurons tells us which neurons are activated in each region. From left to right, we start with $h_3$ activated, then we activate $h_1$ and $h_2$ as we move forward, we deactivate $h_3$, and finally we activate $h_4, \ldots, h_n$ in sequence. This yields the following system of equations, where $t_{n+1}$ is defined as 1 for simplicity:
\begin{align}
w_3 x + (b_3 + d) &= \frac{1}{t_1}x\tag{$R_1$}\\
\left(w_1 + w_3\right) x + \left(b_1 + b_3 + d\right) &= -\frac{1}{t_2 - t_1}x + \frac{t_2}{t_2 - t_1}\tag{$R_2$}\\
\left(w_1 + w_2 + w_3\right) x + \left(b_1 + b_2 + b_3 + d\right) &= \frac{1}{t_3 - t_2}x - \frac{t_2}{t_3 - t_2}\tag{$R_3$}\\
\left(w_1 + w_2\right) x + \left(b_1 + b_2 + d\right) &= -\frac{1}{t_4 - t_3}x + \frac{t_4}{t_4 - t_3}\tag{$R_4$}\\
\left(w_1 + w_2 + \sum_{j=4}^{i-1} w_j\right) x + \left(b_1 + b_2 + \sum_{j=4}^{i-1} b_j + d\right) &= \begin{cases}\frac{1}{t_i - t_{i-1}}x - \frac{t_{i-1}}{t_i - t_{i-1}} \quad \text{if $i$ is odd}\\
-\frac{1}{t_i - t_{i-1}}x + \frac{t_{i}}{t_i - t_{i-1}} \quad \text{if $i$ is even}
\end{cases}\tag{$R_i$}\\&\hspace{9em}\text{for all $i = 5, \ldots, n + 1$}\notag
\end{align}

It is left to show that there exists a solution to this system of linear equations such that $0 < t_1 < \ldots < t_n < 1$.

First, note that all of the biases $b_1, \ldots, b_n, d$ can be written in terms of $t_1, \ldots, t_n$. Note that if we subtract $(R_4)$ from $(R_3)$, we can express $b_3$ in terms of the $t_i$ variables. The remaining equations become triangular, and therefore given any values for $t_i$'s we can back-substitute the remaining bias variables.

The same subtraction yields $w_3$ in terms of $t_i$'s. However, both $(R_1)$ and $(R_3) - (R_4)$ define $w_3$ in terms of the $t_i$ variables, so they must be the same:
\begin{align*}
\frac{1}{t_1} = \frac{1}{t_3 - t_2} + \frac{1}{t_4 - t_3}.
\end{align*}
If we find values for $t_i$'s satisfying this equation and $0 < t_1 < \ldots < t_n < 1$, all other weights can be obtained by back-substitution since eliminating $w_3$ yields a triangular set of equations.

In particular, the following values are valid: $t_1 = \frac{1}{2n + 1}$ and $t_i = \frac{2i - 1}{2n + 1}$ for all $i = 2, \ldots, n$. The remaining weights and biases can be obtained as described above, which completes the desired construction.

As an example, a construction with four units is depicted in Figure~\ref{fig:1d_construction}. Its breakpoints are $t_1 = \frac{1}{9}$, $t_2 = \frac{3}{9}$, $t_3 = \frac{5}{9}$, and $t_4 = \frac{7}{9}$. Its ReLUs are $h_1(x) = \max\{0, -\frac{27}{2} x + \frac{3}{2}\}$, $h_2(x) = \max\{0, 9x-3\}$, $h_3(x) = \max\{0, 9x-5\}$, and $h_4(x) = \max\{0, 9x\}$. Finally, $\tilde{h}(x) = [-1, 1, -1, 1] [h_1(x), h_2(x), h_3(x), h_4(x)]^\top + 5$.

\end{proof}

\section{Proof of Theorem~\ref{thm:lower_bound}}\label{sec:proof_lower_bound}

\begin{refthm}{thm:lower_bound}
The maximal number of linear regions induced by a rectifier network with $n_0$ input units and $L$ hidden layers with $n_l \geq 3n_0$ for all $l$ is lower bounded by
\begin{align*}
\left(\prod_{l=1}^{L-1}\left(\left\lfloor \frac{n_l}{n_0} \right\rfloor + 1\right)^{n_0}\right) \sum_{j=0}^{n_0} \binom{n_L}{j}.
\end{align*}
\end{refthm}
\begin{proof}
We follow the proof of Theorem~5 from \citet{Montufar2014} except that we use a different 1-dimensional construction. The main idea of the proof is to organize the network into $n_0$ independent networks with input dimension 1 each and apply the 1-dimensional construction to each individual network. In particular, for each layer $l$ we assign $\lfloor n_l / n_0 \rfloor$ ReLUs to each network, ignoring any remainder units. In \citet{Montufar2014}, each of these networks have at least $\prod_{l=1}^L \lfloor n_l / n_0 \rfloor$ regions. We instead use Theorem~\ref{thm:dimension_one} to attain $\prod_{l=1}^L (\lfloor n_l / n_0 \rfloor + 1)$ regions in each network. We assume that $n_l \geq 3n_0$ since each of the networks from Theorem~\ref{thm:dimension_one} requires at least 3 units per layer.

Since the networks are independent from each other, the number of activation patterns of the compound network is the product of the number of activation patterns of each of the $n_0$ networks. Hence, the same holds for the number of regions. Therefore, the number of regions of this network is at least $(\prod_{l=1}^L (\lfloor n_l / n_0 \rfloor + 1))^{n_0}$.

In addition, we can replace the last layer by a function representing an arrangement of $n_L$ hyperplanes in general position that partitions $(0, 1)^{n_0}$ into $\sum_{j = 0}^{n_0} \binom{n_L}{j}$ regions. This yields the lower bound of $\prod_{l=1}^{L-1} (\lfloor n_l / n_0 \rfloor + 1)^{n_0} \sum_{j = 0}^{n_0} \binom{n_L}{j}$.

\end{proof}

\section{Proof of Theorem~\ref{thm:lower_bound_arora}}\label{sec:proof_lower_bound_arora}

\begin{refthm}{thm:lower_bound_arora}
For any values of $m \geq 1$ and $w \geq 2$, there exists a rectifier network with $n_0$ input units and $L$ hidden layers of size $2m + w(L-1)$ that has $2\sum_{j=0}^{n_0-1}\binom{m-1}{j}(w+1)^{L-1}$ linear regions.
\end{refthm}
\begin{proof}
Theorem 6.1 and Lemma 6.2 in \citet{Arora2018} imply that for any $m \geq 1$, we can construct a layer representing a function from $\mathbb{R}^n$ to $\mathbb{R}$ with $2m$ ReLUs that has $2\sum_{j=0}^{n_0-1}\binom{m-1}{j}$ regions. Consider the network where this layer is the first one and the remaining layers are the one-dimensional layers from Theorem~\ref{thm:dimension_one}, each of size $w$. Then this network has size $2m + w(L-1)$ and $2\sum_{j=0}^{n_0-1}\binom{m-1}{j}(w+1)^{L-1}$ regions.
\end{proof}

\section{Proof of Theorem~\ref{thm:upper_bound_maxout}}\label{sec:proof_upper_bound_maxout}

\begin{refthm}{thm:upper_bound_maxout}
Consider a deep neural network with $L$ layers, $n_l$ rank-$k$ maxout units at each layer $l$, and an input of dimension $n_0$. The maximal number of regions of this neural network is at most
\begin{align*}
\prod_{l=1}^L \sum_{j=0}^{d_l} \binom{\frac{k(k-1)}{2} n_l}{j}
\end{align*}
where $d_l = \min\{n_0, n_1, \ldots, n_l\}$.

Asymptotically, if $n_l = n$ for all $l = 1, \ldots, L$, $n \geq n_0$, and $n_0 = O(1)$, then the maximal number of regions is at most $O((k^2n)^{Ln_0})$.
\end{refthm}
\begin{proof}
We denote by $W_j^l$ the $n_l \times n_{l-1}$ matrix where the rows are given by the $j$-th weight vectors of each rank-$k$ maxout unit at layer $l$, for $j = 1, \ldots, k$. Similarly, $b_j^l$ is the vector composed of the $j$-th biases at layer $l$.

In the case of maxout, an activation pattern $\mathcal{S} = (S^1, \ldots, S^l)$ is such that $S^l$ is a vector that maps from layer-$l$ neurons to $\{1, \ldots, k\}$. We say that the activation of a neuron is $j$ if $w_j x + b_j$ attains the maximum among all of its functions; that is, $w_j x + b_j \geq w_{j'} x + b_{j'}$ for all $j' = 1, \ldots, j$. In the case of ties, we assume the function with lowest index is considered as its activation.

Similarly to the ReLU case, denote by $\phi_{S^l} : \mathbb{R}^{n_l \times n_{l-1} \times k} \to \mathbb{R}^{n_l \times n_{l-1}}$ the operator that selects the rows of $W_1^l, \ldots, W_k^l$ that correspond to the activations in $S^l$. More precisely, $\phi_{S^l}(W_1^l, \ldots, W_k^l)$ is a matrix $W$ such that its $i$-th row is the $i$-th row of $W_j^l$, where $j$ is the neuron $i$'s activation in $S^l$. This essentially applies the maxout effect on the weight matrices given an activation pattern.

\citet{Montufar2014} provides an upper bound of $\sum_{j=0}^{n_0} \binom{k^2 n}{j}$ for the number of regions for a single rank-$k$ maxout layer with $n$ neurons. The reasoning is as follows. For a single maxout unit, there is one region per linear function. The boundaries between the regions are composed by segments that are each contained in a hyperplane. Each segment is part of the boundary of at least two regions and conversely each pair of regions corresponds to at most one segment. Extending these segments into hyperplanes cannot decrease the number of regions. Therefore, if we now consider $n$ maxout units in a single layer, we can have at most the number of regions of an arrangement of $k^2 n$ hyperplanes. In the results below we replace $k^2$ by $\binom{k}{2}$, as only pairs of distinct functions need to be considered.

We need to define more precisely these $\binom{k}{2}n$ hyperplanes in order to apply a strategy similar to the one from the Section~\ref{sec:upper_bound_rectifier}. In a single layer setting, they are given by $w_j x + b_j = w_{j'} + b_{j'}$ for each distinct pair $j, j'$ within a neuron. In order to extend this to multiple layers, consider a $\binom{k}{2}n_l \times n_{l-1}$ matrix $\hat{W}_l$ where its rows are given by $w_j - w_{j'}$ for every distinct pair $j, j'$ within a neuron $i$ and for every neuron $i = 1, \ldots, n_l$. Given a region $\mathcal{S}$, we can now write the weight matrix corresponding to the hyperplanes described above: $\hat{W}^l_{\mathcal{S}} := \hat{W}^l \; \phi_{S^{l-1}}(W_1^{l-1}, \ldots, W_k^{l-1}) \cdots \phi_{S^1}(W_1^1, \ldots, W_k^1)$. In other words, the hyperplanes that extend the boundary segments within region $\mathcal{S}$ are given by the rows of $\hat{W}^l_{\mathcal{S}} x + b = 0$ for some bias $b$.

A main difference between the maxout case and the ReLU case is that the maxout operator $\phi$ does not guarantee reductions in rank, unlike the ReLU operator $\sigma$. We show the analogous of Lemma~\ref{lem:regions_rank_network} for the maxout case. However, we fully relax the rank.

\begin{lem}\label{lem:regions_rank_network_maxout}
The number of regions induced by the $n_l$ neurons at layer $l$ within a certain region $\mathcal{S}$ is at most $\sum_{j=0}^{d_l} \binom{\frac{k(k-1)}{2} n_l}{j}
$, where $d_l = \min\{n_0, n_1, \ldots, n_l\}$.
\end{lem}
\begin{proof}
For a fixed region $\mathcal{S}$, an upper bound is given by the number of regions of the hyperplane arrangement corresponding to $\hat{W}^l_{\mathcal{S}} x + b = 0$ for some bias $b$. The rank of $\hat{W}^l_{\mathcal{S}}$ is upper bounded by
\begin{align*}
\operatorname{rank}(\hat{W}^l_{\mathcal{S}}) &= \operatorname{rank}(\hat{W}^l \; \phi_{S^{l-1}}(W_1^{l-1}, \ldots, W_k^{l-1}) \cdots \phi_{S^1}(W_1^1, \ldots, W_k^1))\\
& \leq \min\{\operatorname{rank}(\hat{W}^l), \operatorname{rank}(\phi_{S^{l-1}}(W_1^{l-1}, \ldots, W_k^{l-1})), \ldots, \operatorname{rank}(\phi_{S^1}(W_1^1, \ldots, W_k^1))\}\\
& \leq \min\{n_0, n_1, \ldots, n_l\}.
\end{align*}
Applying Lemma~\ref{lem:regions_rank} yields the result.
\end{proof}

Since we can consider the partitioning of regions independently from each other, Lemma~\ref{lem:regions_rank_network_maxout} implies that the maximal number of regions of a rank-$k$ maxout network is at most $\prod_{l=1}^L \sum_{j=0}^{d_l} \binom{\frac{k(k-1)}{2} n_l}{j}$ where $d_l = \min\{n_0, n_1, \ldots, n_l\}$.

\end{proof}

\section{Implementation of the mixed-integer formulation}\label{sec:mir_impl}

In practice, the value of constant $M$ should be chosen to be as small as possible, 
which also implies choosing different values on different places 
to make the formulation tighter and more stable numerically~\citep{BigM}. 
For the constraints set~\eqref{cons:hyp}--\eqref{cons:zb}, 
it suffices to choose $M$ to be as large as either $h_i^l$ or $\bar{h}_i^l$ can be 
given the bounds on the input. 
Hence, 
we can respectively replace $M$ with $H^l_i$ and $\bar{H}^l_i$ in the constraints involving those variables. 
If we are given lower and upper bounds for $X$, which we can use for $H^0$ and $\bar{H}^0$, 
then we can define subsequent bounds as follows: 
\[
H^l_i = \max\left\{0, \sum_{j} \max\left\{0, w_{i j}^l H^{l-1}_j\right\} + b_i^l \right\}
\]  
\[
\overline{H}^l_i = \max\left\{0, \sum_{j} \max\left\{0, - w_{i j}^l H^{l-1}_j\right\} - b_i^l \right\}
\]
For the constraint involving $f$ in formulation~$\mathcal{P}$, 
we should choose a slightly larger value than $H^l_i$ for correctness because some neurons may never be active within the input bounds.

\section{Mixed-integer representability of rectifier networks}\label{sec:mir}

\begin{cor}
If the input $X$ is a polytope, 
then the $(x,y)$ mapping of a rectifier DNN is mixed-integer representable.
\end{cor}
\begin{proof}
Immediate from the existence of a mixed-integer formulation mapping $x$ to $y$, 
which is correct as long as the input is bounded and thus a sufficiently large $M$ exists.
\end{proof}

Formulation $\mathcal{P}$ and the result above have important consequences. First, they allow us to tap into the literature of mixed-integer representability~\citep{MIR} and disjunctive programming~\citep{DP} to understand what can be modeled on rectifier networks with a finite number of neurons and layers. To the best of our knowledge, that has not been discussed before. Second, they imply that we can use mixed-integer optimization solvers to analyze the $(\mathbf{x},\mathbf{y})$ mapping of a trained neural network. For example, 
\citet{Cheng2017} use another mixed-integer formulation to generate adversarial examples. 

\section{Counting linear regions of ReLUs with unrestricted inputs} \label{sec:unrestricted_inputs}

More generally, 
we can represent linear regions as a disjunctive program~\citep{DP}, 
which consist of a union of polyhedra. Disjunctive programs are used in the integer programming literature to generate cutting planes by lift-and-project~\citep{CGLP}. 
In what follows, we assume that a neuron can be either active or inactive when the output lies on the activation hyperplane. 

For each active neuron, 
we can use the following constraints to map input to output:
\begin{align}
w^l_i h^{l-1} + b^l_i = h^l_i \label{cons:act_hyp} \\ 
h^l_i \geq 0 \label{cons:act_hp} 
\end{align}
For each inactive neuron, we use the following constraint:
\begin{align}
w^l_i h^{l-1} + b^l_i \leq 0 \label{cons:ina_hyp} \\
h^l_i = 0 \label{cons:ina_hp} 
\end{align}

\begin{thm}
The set of linear regions of a rectifier network is a union of polyhedra.
\end{thm}
\begin{proof}
First, the activation set $S^l$ for each level $l$ defines the following mapping:
\begin{align}\label{dp}
\bigcup_{S^l \subseteq \{1, \ldots, n_l \}, l \in \{1, \ldots, L+1\}} \big\{ (h^0, h^1, \ldots, h^{L+1}) ~ | ~ \eqref{cons:act_hyp}-\eqref{cons:act_hp} \text{ if } i \in S^l; \eqref{cons:ina_hyp}-\eqref{cons:ina_hp} \text{ otherwise } \big\}
\end{align}
Consequently, we can project the variables sets $h^1, \ldots, h^{L+1}$ out of each of those terms by Fourier-Motzkin elimination~\citep{Fourier}, 
thereby yielding a polyhedron for each combination of active sets across the layers.
\end{proof}

Note that the result above is similar in essence to Theorem~2 of~\cite{Raghu2017}.

\begin{cor}
If $X$ is unrestricted, then the number of linear regions can be counted using $\mathcal{P}$  
if $M$ is large enough. 
\end{cor}
\begin{proof}
To count regions, we only need one point $x$ from each linear region. 
Since the number of linear regions is finite, 
then it suffices if $M$ is large enough to correctly map a single point in each region. 
Conversely, each infeasible linear region either corresponds to empty sets of~\eqref{dp} 
or else to a polyhedron $P$ such that $\{ (h^1, \ldots, h^{L+1}) \in P ~ | ~ h_i^l > 0 ~ \forall l \in \{1, \ldots, L+1\}, i \in S^l \big\}$ is empty, 
and neither case would yield a solution for the $z$-projection of $\mathcal{P}$. 
\end{proof}

\section{Mixed-integer representability of maxout units} \label{sec:mir_maxouts}

In what follows, we assume that we are given a neuron $i$ in level $l$ with output $h^l_i$.  
For that neuron, we denote the vector of weights as $w^{l i}_1, \ldots, w^{l i}_k$. 
Thus, the neuron output corresponds to 
\[
h_i^l := \max \left\{ w_1^{l i} h^{l-1} + b_1, \ldots, w_k^{l i} h^{l-1} + b_k \right\}
\]

Hence, we can connect inputs to outputs for that given neuron as follows:

\begin{align}
w^{l i}_{j} h^{l-1}_j + b^{l i}_j = g^{l i}_j, & \qquad j = 1, \ldots, k  \label{max:hyp} \\ 
h^l_i \geq g^{l i}_j, & \qquad j = 1, \ldots, k \label{max:hp} \\
h^l_i \leq g^{l i}_j + M (1- z^{l i}_j) & \qquad j = 1, \ldots, k \label{max:hu} \\
z^{l i}_j \in \{0, 1\}, & \qquad j = 1, \ldots, k \label{max:zb} \\
\sum_{j=1}^k z^{l i}_j = 1 \label{max:sumz} 
\end{align}

The formulation above generalizes that for ReLUs with some small modifications. 
First, we are computing the output of each term with constraint~\eqref{max:hyp}. 
The output of the neuron is lower bounded by that of each term with constraint~\eqref{max:hp}. 
Finally, we have a binary variable $z^{l i}_m$ per term of each neuron, 
which denotes which term matches the output. 
Constraint~\eqref{max:sumz} enforces that only one variable is at one per neuron, 
whereas constraint~\eqref{max:hu} equates the output of the neuron with the active term. 
Each constant $M$ should be chosen in a way that the other terms can vary freely, 
hence effectively disabling the constraint when the corresponding binary variable is at zero.

\section{Runtimes for counting the linear regions} \label{sec:runtimes}

Table~\ref{tab:runtimes} reports the median runtimes to count different configurations of networks on the experiment.

\begin{table}[!h]
\[
\begin{array}{ccc}
&\textbf{Network configuration} & \textbf{Runtime (s)} \\ 
\hline
& 1;21;10 & 5.3 \times 10^{-1} \\
& 2;20;10 & 8.0 \times 10^{-1} \\
& 3;19;10 & 5.0 \times 10^{0} \\
& 4;18;10 & 4.0 \times 10^1 \\
& 5;17;10 & 2.1 \times 10^2 \\
& 6;16;10 & 5.7 \times 10^2 \\
& 7;15;10 & 1.8 \times 10^3 \\
& 8;14;10 & 4.5 \times 10^3 \\
& 9;13;10 & 9.2 \times 10^3 \\
& 10;12;10 & 1.7 \times 10^4 \\
& 11;11;10 & 3.3 \times 10^4 \\
& 12;10;10 & 5.3 \times 10^4 \\
& 13;9;10 & 7.3 \times 10^4 \\
& 14;8;10 & 1.1 \times 10^5 \\
& 15;7;10 & 9.3 \times 10^4 \\
& 16;6;10 & 1.3 \times 10^5 \\
& 17;5;10 & 1.6 \times 10^5 \\
& 18;4;10 & 1.8 \times 10^5 \\
& 19;3;10 & 2.4 \times 10^5 \\
& 20;2;10 & 9.9 \times 10^4 \\
& 21;1;10 & 3.7 \times 10^4 \\
\end{array}
\]
\caption{\emph{Median runtimes for counting the trained networks for each configuration used in the experiment.}}
\label{tab:runtimes}
\end{table}

\end{document}